%% file: main.tex
\documentclass{article}

\usepackage{microtype}
\usepackage{graphicx}
\usepackage{subcaption}
\usepackage{booktabs} % for professional tables

\usepackage{hyperref}

\usepackage[preprint]{icml2026}

\usepackage{amsmath}
\usepackage{amssymb}
\usepackage{mathtools}
\usepackage{amsthm}

\usepackage[capitalize,noabbrev]{cleveref}

\theoremstyle{plain}
\newtheorem{theorem}{Theorem}[section]

\newtheorem{lemma}[theorem]{Lemma}
\newtheorem{corollary}[theorem]{Corollary}
\theoremstyle{definition}
\newtheorem{definition}[theorem]{Definition}

\theoremstyle{remark}

\usepackage[textsize=tiny]{todonotes}

\usepackage{adjustbox}
\usepackage{wrapfig}
\usepackage{enumitem}
\usepackage[table]{xcolor}
\usepackage{caption}
\usepackage{multirow}

\usepackage{xspace}
\newcommand{\mName}{{DSGD}\xspace}

\icmltitlerunning{Dynamic Scaled Gradient Descent for Stable Fine-Tuning for Classifications}

\begin{document}

\twocolumn[
  \icmltitle{Dynamic Scaled Gradient Descent for Stable Fine-Tuning for Classifications}



  \icmlsetsymbol{equal}{*}

  \begin{icmlauthorlist}
    \icmlauthor{Nghia Bui}{njit}
    \icmlauthor{Lijing Wang}{njit}
  \end{icmlauthorlist}

  \icmlaffiliation{njit}{Department of Data Science, New Jersey Institute of Technology, New Jersey, USA}

  \icmlcorrespondingauthor{Lijing Wang}{lijing.wang@njit.edu}

  \icmlkeywords{Large Pretrained Model, Fine-Tuning, Learning Stability, Random Seeds, Gradient Cancellation, Collapsed Failure}

  \vskip 0.3in
]



\printAffiliationsAndNotice{}  % no special notice (required even if empty)

\begin{abstract}
Fine-tuning pretrained models has become a standard approach to adapting pretrained knowledge to improve the accuracy on new sparse, imbalance datasets. However, issues arise when optimization falls into a collapsed state, where the model gets stuck, leading to degraded performance and unstable training. One possible reason for this is the cancellation of gradients across training examples. To address this problem, we propose a novel algorithm, dynamic scaled gradient descent (\mName), that directly modifies the gradients returned by training examples, specifically, scaling down the gradients of correctly classified examples using a dynamic scaler. This strategy offers both theoretical and empirical advantages in improving training stability. Experiments on a variety of benchmark datasets, spanning multiple tasks and large pretrained models, demonstrate that our method consistently reduces performance variance and surpasses the accuracy of existing approaches.
\end{abstract}

\input{sections/introduction}

\input{sections/problem}

\input{sections/method}

\input{sections/theory}

\input{sections/experiment}

\input{sections/conclusion}

% \section*{Acknowledgements}

% \textbf{Do not} include acknowledgements in the initial version of the paper

% If a paper is accepted, the final camera-ready version can (and usually should)
% include acknowledgements.  Such acknowledgements should be placed at the end of
% the section, in an unnumbered section that does not count towards the paper
% page limit. Typically, this will include thanks to reviewers who gave useful
% comments, to colleagues who contributed to the ideas, and to funding agencies
% and corporate sponsors that provided financial support.

\section*{Impact Statement}
This paper presents work whose goal is to advance the field of Machine
Learning. There are many potential societal consequences of our work, none
which we feel must be specifically highlighted here.


\bibliography{ref}
\bibliographystyle{icml2026}

%%%%%%%%%%%%%%%%%%%%%%%%%%%%%%%%%%%%%%%%%%%%%%%%%%%%%%%%%%%%%%%%%%%%%%%%%%%%%%%
%%%%%%%%%%%%%%%%%%%%%%%%%%%%%%%%%%%%%%%%%%%%%%%%%%%%%%%%%%%%%%%%%%%%%%%%%%%%%%%
% APPENDIX
%%%%%%%%%%%%%%%%%%%%%%%%%%%%%%%%%%%%%%%%%%%%%%%%%%%%%%%%%%%%%%%%%%%%%%%%%%%%%%%
%%%%%%%%%%%%%%%%%%%%%%%%%%%%%%%%%%%%%%%%%%%%%%%%%%%%%%%%%%%%%%%%%%%%%%%%%%%%%%%
\newpage
\appendix
\onecolumn
\input{sections/appendix}
%%%%%%%%%%%%%%%%%%%%%%%%%%%%%%%%%%%%%%%%%%%%%%%%%%%%%%%%%%%%%%%%%%%%%%%%%%%%%%%
%%%%%%%%%%%%%%%%%%%%%%%%%%%%%%%%%%%%%%%%%%%%%%%%%%%%%%%%%%%%%%%%%%%%%%%%%%%%%%%

\end{document}

%% file: sections/introduction.tex
\section{Introduction}\label{sec:introduction}

The blossom of deep learning in recent years has led to the success of large pretrained models (LPMs), such as BERT~\cite{devlin2019bert}, RoBERTa~\cite{liu2019roberta}, LLaMA~\cite{touvron2023llama} in natural language processing (NLP) and ViT in computer vision~\cite{dosovitskiy2020image}, where they can be finetuned for a wide range of downstream tasks with remarkable performance.
Despite their success, fine-tuning LPMs often suffers from high sensitivity to random seeds~\cite{dodge2020fine,mosbach2020stability, zhang2020revisiting,cao2019learning}, a phenomenon known as training instability, causing significant performance variation across runs~\cite{bethard2022we,hua2021noise,wang-etal-2023-two,bui2025assessing}, as shown in Figure~\ref{fig:collapsed_proportion}.
As a result, obtaining reliable performance requires multiple training runs, a process that is both computationally prohibitive and impractical for real-time applications.
Training stability is a well-documented challenge, particularly during fine-tuning, where pretrained representations become sensitive to factors such as weight initialization, data ordering, and hyperparameter choices. Prior work has studied this phenomenon in NLP~\cite{mosbach2020stability, zhang2020revisiting, dodge2020fine, phang2018sentence, wang-etal-2023-two, nishida-etal-2025-instability, hidey-etal-2022-reducing,bui2025assessing} and computer vision~\cite{summers2021nondeterminism, pecher-etal-2024-fighting}, attributing instability to distribution shifts between pretraining and downstream data, as well as optimization nondeterminism in deep learning pipelines~\cite{picard2021torch, hardt2016train, mosbach-2023-analyzing, zhang2020revisiting}.

\begin{figure}[t]
    \centering
    \includegraphics[width=\linewidth]{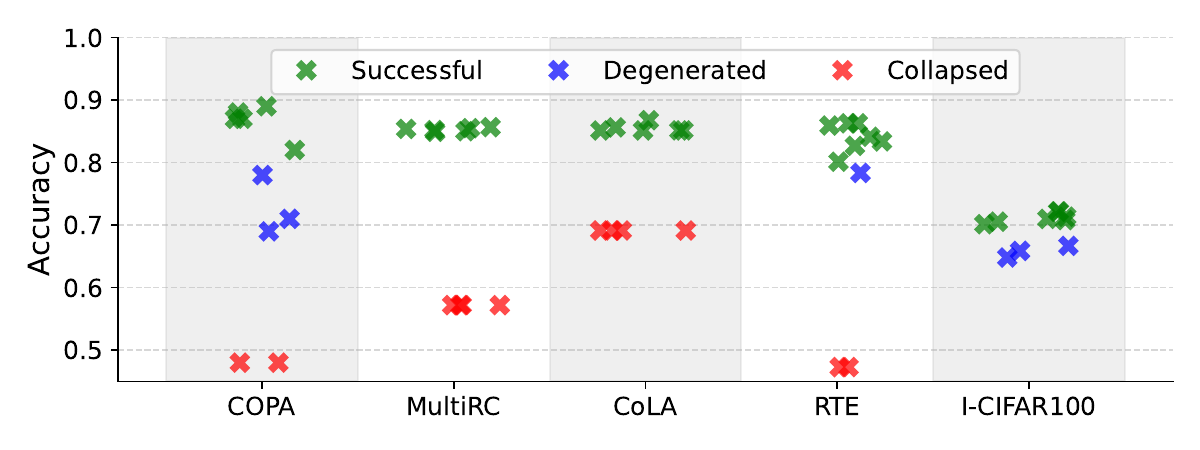}
    \caption{Training failures (either collapsed state or degenerate solution) on various tasks when fine-tuning pretrained models across 10 random seeds. Each task uses the same hyperparameters, with only the random seed varying between runs.}
    \label{fig:collapsed_proportion}
\end{figure}

To reduce performance variation induced by random seeds, few methods have been proposed. While ensembles and their variations~\cite{wang-etal-2023-two, wang2020wisdom, nishida-etal-2025-instability,izmailov2018averaging, nishida-etal-2025-instability} are theoretically grounded and often the most effective at lowering variance, they are computationally (either training time or storage) prohibitive for large-scale deployment. Data-centric solutions like increasing dataset size \cite{dodge2020fine} are often impractical, and optimization guidelines such as extended training \cite{mosbach2020stability} still result in high variance. 
Noise-based techniques~\cite{hua2021noise, wu2022noisytune} typically seek to improve generalization and stability of LPMs in fine-tuning by injecting noise but lack a targeted mechanism to directly stabilize the fine-tuning optimization process and can increase divergence.
Consequently, stabilizing fine-tuning of LPMs remains largely a process of trial and error, with the core issue of seed-induced variance still unresolved.

In this work, we attribute fine-tuning instability to gradient conflicts across training examples. Under certain random seeds, conflicting gradients arising from parameter initialization or data ordering may cancel each other, causing premature convergence even when training error remains high. Consequently, a model fails to learn meaningful patterns, resulting in drastically different outputs across runs.
Although no prior work has been directly working on solving gradient cancellation across training examples, the gradient conflict resolvers from multitask learning \cite{sener2018multi, yu2020gradient, liu2021conflict, chen2020just} were proposed to reduce inter-task gradient conflict through projection or weighting schemes; however, these methods require expensive separate backward passes and rely on detecting significant negative cosine similarity between gradients, failing when gradients are orthogonal or small-magnitude. 

We propose a \textbf{D}ynamic \textbf{S}caled \textbf{G}radient \textbf{D}escent (\textbf{\mName}) algorithm for classification that explicitly targets optimization instability during fine-tuning. We adaptively scale gradients from correctly classified examples to alleviate gradient cancellation, leading to improved convergence and fine-tuning stability under seed-induced randomness. We also provide theoretical guarantee and empirical evidence to validate the proposed algorithm. 
\mName is a drop-in replacement for standard optimizers: \textit{no additional backward passes, no architectural changes, no task-specific tuning}, and \textit{negligible overhead}. 
To our knowledge, \textbf{\mName is the first method to provably and empirically avoid seed-induced collapsed failure during gradient-based learning}.

\textbf{Major Contributions:} 
\textbf{(1)} identifies gradient conflicts between correctly and wrongly classified examples as the key cause of a failed run during fine-tuning; 
\textbf{(2)} proposes \mName, a lightweight and proactive optimization algorithm that dynamically downscales the gradient contributions of correctly classified examples within each mini-batch to mitigate gradient conflicts; 
\textbf{(3)} provides a formal theoretical motivation for \mName via gradient decomposition; 
% Analysis shows that preserving gradients from wrongly classified examples while modulating gradients from correctly classified ones aligns with theoretical optimality; 
\textbf{(4)} proves \mName ensures convergence to better stationary points and a tighter stability upper bound than standard gradient descent, formally guaranteeing improved accuracy and stability; 
\textbf{(5)} demonstrates consistent and significant improvements in accuracy and stability over baselines across 14 diverse NLP and vision tasks, confirming broad applicability.

%% file: sections/problem.tex
\begin{figure*}[t]
    \centering
    \begin{subfigure}{0.24\linewidth}
        \centering
        \includegraphics[width=\linewidth]{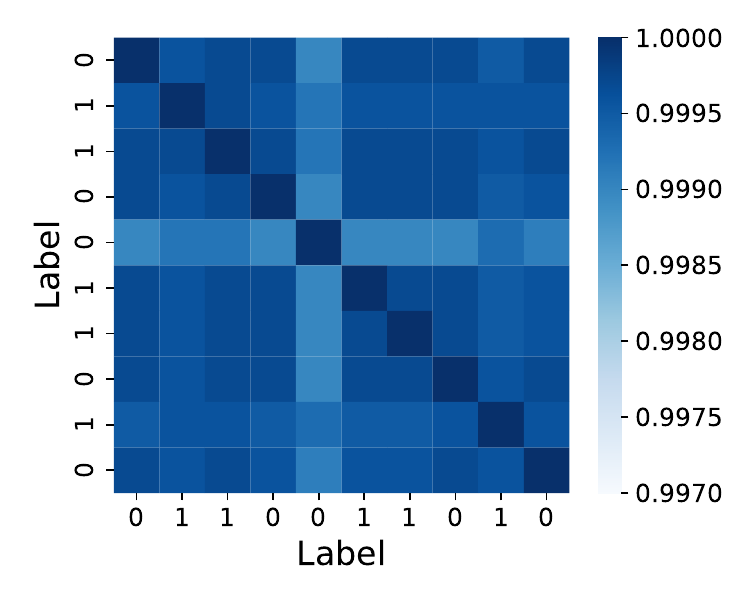}
        \caption{Collapsed state.}
        \label{fig:sim_feat}
    \end{subfigure}
    \begin{subfigure}{0.24\linewidth}
        \centering
        \includegraphics[width=\linewidth]{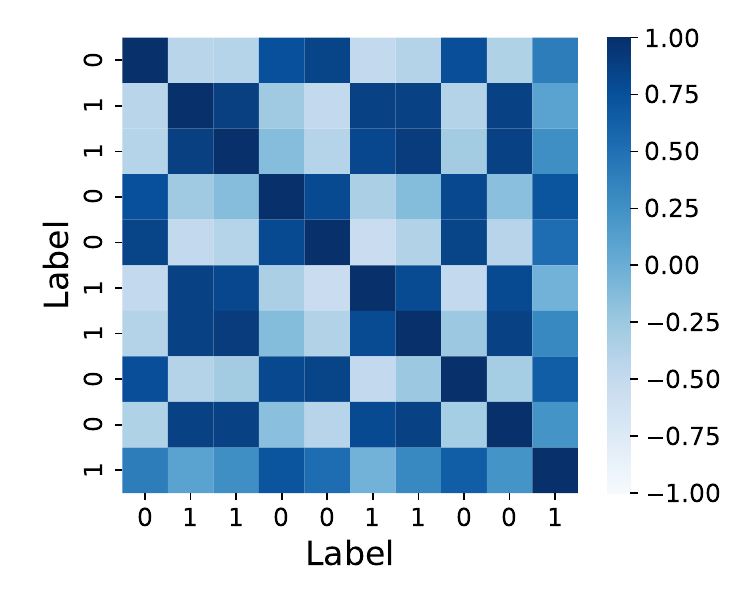}
        \caption{Successful run.}
        \label{fig:sim_feat_s}
    \end{subfigure}
    \begin{subfigure}{0.24\linewidth}
        \centering
        \includegraphics[width=\linewidth]{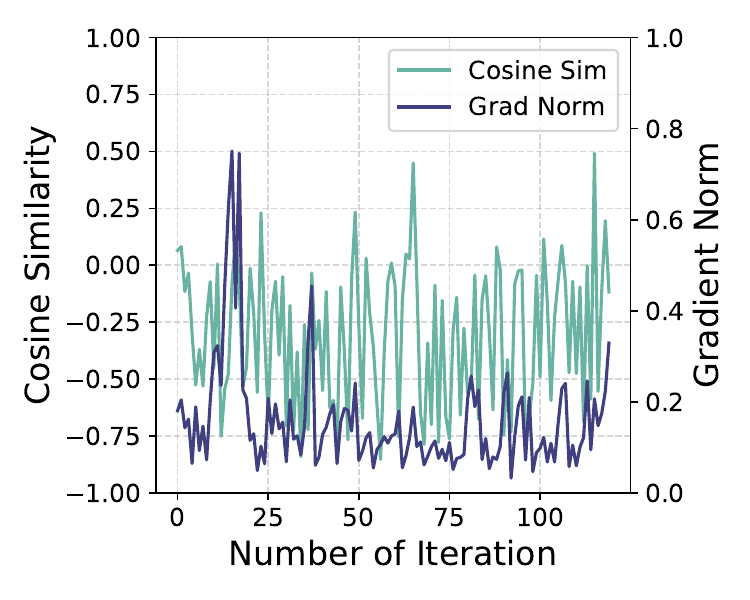}
        \caption{Failed run.}
        \label{fig:cosine_conflict}
    \end{subfigure}
    \begin{subfigure}{0.24\linewidth}
        \centering
        \includegraphics[width=\linewidth]{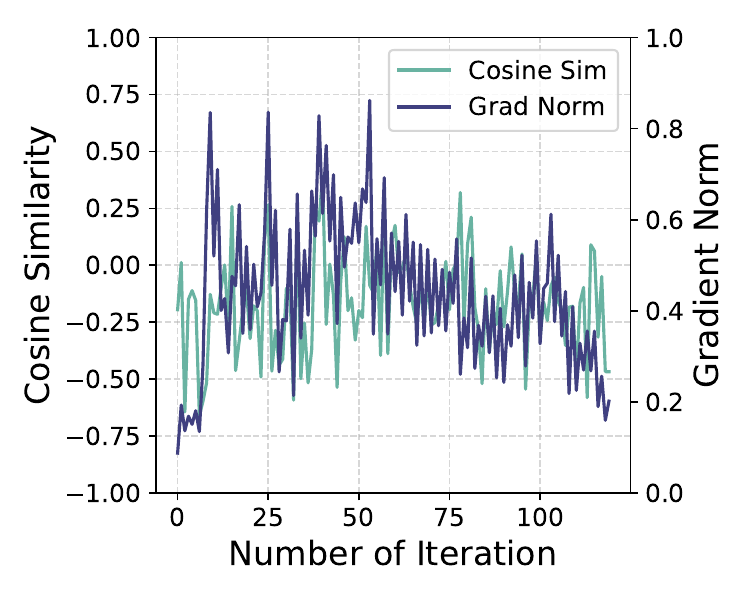}
        \caption{Successful run.}
        \label{fig:cosine_conflict_suc}
    \end{subfigure}
    \caption{(a) and (b) are cross-cosine similarity of \texttt{[CLS]} token representations for 10 randomly selected examples (5 per class) when fine-tuning \texttt{\small RoBERTa-large} on COPA (a binary classification task), in a collapsed state (random seed 132) and a successful run (random seed 42), respectively. Axes list the corresponding labels. Representations in collapsed states become uniformly similar, while successful runs show higher within-class than between-class similarity. (c) and (d) show gradient norms of all training examples and cosine similarities between gradients of correctly and incorrectly classified training examples within a batch in a failed run and successful run respectively, measured by mean of 10 randomly selected dense layers over training iterations. }
    \label{fig:failed_suc_comparison}
\end{figure*}

\section{Phenomenon}\label{sec:problem}
% This section defines key terminology for the paper, our discussion and analysis focus on classifications.

\begin{definition}
\label{def:convergence}
    \textbf{Convergence:} The state during training when model parameters stabilize, with the loss function exhibiting negligible further decrease. Convergence may result in either a high-performing solution, a degenerate solution, or a collapsed state, depending on the optimization trajectory and loss landscape.
\end{definition}

\begin{definition}
\label{def:degenerate-solution}
    \textbf{Degenerate Solution:} A suboptimal outcome of optimization where the model produces varied predictions but fails to learn meaningful decision boundaries, resulting in poor generalization performance despite non-trivial predictive behavior.
\end{definition}

\begin{definition}
\label{def:collapsed-state}
    \textbf{Collapsed State:} A pathological failure mode of optimization where the model's predictions saturate to a single class label for all inputs, representing a severe form of model collapse with zero predictive variance.
\end{definition}

\begin{definition}
\label{def:gradient-cancellation}
    \textbf{Gradient Cancellation:} A phenomenon occurring during optimization where the expected gradient of training examples within a batch can be approximately zero, leading to no effective update.
\end{definition}

\textbf{Phenomenon:} As shown in Figure~\ref{fig:collapsed_proportion}, when fine-tuning a LPM, certain random seeds can lead to a \textit{training failure}, either a collapsed state (red crosses) or a degenerate solution (blue crosses). This sensitivity results in significant performance variance across different runs. 

\textbf{Underlying phenomenon:} To investigate the reasons for a training failure, we compute the pairwise cosine similarity matrix of \texttt{[CLS]} token representations from a random subset of training examples after fine-tuning. The results are shown in Figure~\ref{fig:sim_feat} for a collapsed state and \ref{fig:sim_feat_s} for a successful run. We can see that in a collapsed state, the \texttt{[CLS]} representations lose their class-discriminative structure and converge to nearly identical vectors, resulting in uniformly high similarity scores. This indicates a failure to learn meaningful feature separation.

Furthermore, in Figure~\ref{fig:cosine_conflict} and ~\ref{fig:cosine_conflict_suc}, we analyze the gradient norm (blue curve) of the within-batch training examples and the cosine similarity (green curve) between the gradients of correctly and wrongly classified examples. Consistent with the observation of~\cite{mosbach2020stability} that gradient norms are markedly lower in failed runs, our results show a similar trend (i.e., close to 0 in most iterations). This phenomenon can be attributed to gradient cancellation.
Cosine similarity reveals substantial gradient conflicts in the failed run, where 90\% of iterations yield negative values and over 30\% fall below -0.5 (as low as -0.85). In contrast, successful run exhibits milder conflicts, with most similarities above -0.25 and only a few below -0.5. 
This consistent opposition in failed runs results in gradient cancellation during the update step of parameters. This widespread cancellation in the early training can lead to a collapsed state or a degenerate solution.

%% file: sections/method.tex
\section{Proposed Solution: \mName}
\label{sec:solution}
\textbf{Motivation:} 
We consider a learning in the context of minibatch training, where within a batch the distribution of classes of examples can be highly imbalanced. This frequently occurs either due to inherent dataset imbalance or simply the randomness of mini-batch sampling~\cite{dodge2020fine}. In such a scenario, models can collapse to a trivial solution where predictions are dominated by the majority class within the batch. In this collapsed state, gradients from the correctly classified examples form a dominant signal that overwhelms and cancels out the crucial error signal from misclassified minority examples. This prevents the model from learning meaningful features for minority classes. This will repeat for each batch training leading to the model trapped in the collapsed state. 
To prevent this, we must break the gradient cancellation, especially early in training.

\textbf{Solution:} 
Assume that we are given $N$ data points in the training dataset $S$: $z_1, z_2, \dots z_N$ where $z_i= (x_i, y_i)$ drawn i.i.d from an unknown distribution $\mathcal{D}$, where $x_i \in \mathcal{X}$ is an input and $y_i \in \mathcal{Y}$ is the corresponding one-hot encoded label. Let $f_\theta: \mathcal{X} \to \mathcal{Y}$ be a model parameterized by  $\theta \in \Theta$, and let $\mathcal{L}(\theta, z)$ be the loss function. Let $\theta_t$ denote model parameters at training iteration $t$, $\eta$ the learning rate, and $\mathcal{B} \subset \mathcal{D}$ a batch of $m$ examples. The standard gradient descent (GD) update at iteration $t$ is $\theta_{t+1} = \theta_t - \eta \frac{1}{m} G$, where 
$G = \sum_{i=1}^m \nabla_\theta \mathcal{L}_i$ denotes the total gradient over the entire batch. 
We define $\mathcal{B}_W$ and $\mathcal{B}_C$ as the subset of wrongly and correctly classified examples in $\mathcal{B}$ respectively. 
We introduce a \textbf{D}ynamical \textbf{S}caled \textbf{G}radient \textbf{D}escent (\textbf{\mName}) algorithm that reduces the influence of correctly classified examples within each batch. Our modified update rule is as follows:

\begin{equation}
    \label{equ:gradscale}
    \theta_{t+1} = \theta_t - \eta \frac{1}{m}\left( G_W + \gamma_{t} G_C \right),
\end{equation}
where $G_W$ and $G_C$ are the total gradients from $\mathcal{B}_W$ and $\mathcal{B}_C$ respectively.
The dynamic scaling factor $\gamma_t \in [0, \tau]$ is defined as a function:
\( \gamma_t = \zeta(t, T, \tau)\), where $t$ denotes the current training iteration, $T$ is the total number of training iterations, and $\tau \leq 1$ 
is a hyperparameter that controls the maximum allowed scaling. A smaller $\tau$ imposes 
a stricter upper bound, thereby more aggressively downscaling the gradient contribution 
from correctly classified examples. 
Various $\gamma_t$ functions are implemented in Section~\ref{subsec:implementation}.

%% file: sections/theory.tex
\section{Theoretical Foundations}
\label{sec:theory}
In this section, we analyze the theoretical properties of \mName. We prove that, under mild assumptions, our method ensures that the optimization process does not stall at poor solutions but instead converges to a more desirable stationary point. Moreover, we prove that our method provides a tighter stability bound than standard gradient descent (GD).
Proofs of the theorems in this section can be found in Appendix~\ref{apx:prfs}.

\subsection{Theoretical Motivation}
Consider a single layer neural network, let $p_i = \sigma(o_i)$ where $o_i = \omega x_i + b$ denote the output of the model, where $\theta = \{\omega, b\}$ are learnable parameters and $\sigma(\cdot)$ is the softmax function.  
$\mathcal{L}$ is the linear combination loss decomposed as:
\begin{equation}
    \label{equ:loss_decompose}
    \mathcal{L} =  \mathcal{L}_C + \mathcal{L}_W ,
\end{equation}
where $\mathcal{L}_C$ and $\mathcal{L}_W$ denote the loss of $\mathcal{B}_C$ and $\mathcal{B}_W$, respectively.
Assume the model parameter is updated using stochastic GD (SGD) with batch size $m=1$~\cite{hardt2016train} and $S$ is linear separable.
We begin by establishing key properties of the two loss components, $\mathcal{L}_W$ and $\mathcal{L}_C$.

\begin{definition}[Zero Training Error]
\label{def:zero-training-error}
A classifier parametrized by $\theta$ is said to attain \emph{zero training error} on $S = \{(x_i, y_i)\}_{i=1}^N$ if $ \forall i,\arg\max\{y_i\} = \arg\max\{p_i\}$.
\end{definition}

\begin{lemma}[Property of $\mathcal{L}_W$] \label{lem:property_lw}
If gradient-based optimization converges to a stationary point of $\mathcal{L}_{W}$, then the resulting classifier attains zero training error (which does not require zero training loss). Consequently, optimizing $\mathcal{L}_{W}$ leads to a parameter $\theta^{*}$ for which every training example is correctly classified.
\end{lemma}

\begin{lemma}[Property of $\mathcal{L}_C$] \label{lem:property_lc}
Convergence to a stationary point of $\mathcal{L}_{C}$ does \emph{not} guarantee zero training error. The resulting classifier may misclassify some training examples.
\end{lemma}

In lemma~\ref{lem:property_lw} and~\ref{lem:property_lc} We can rewrite $\mathcal{L}_C$ and $\mathcal{L}_W$ in term of weighted sums of individual training point loss to ensure the continuity over the training data partitions regardless they are empty or not.

\textbf{Remark:} The arguments in Lemma~\ref{lem:property_lw} and Lemma~\ref{lem:property_lc} depend only on per-example first-order optimality and thus extend straightforwardly to multi-layer neural networks.

\begin{corollary}[Motivation for Gradient Modulation] \label{cor:motivation}
Since optimizing $\mathcal{L}_{W}$ guarantees zero training error (Lemma~\ref{lem:property_lw}) while optimizing $\mathcal{L}_{C}$ does not (Lemma~\ref{lem:property_lc}), it is theoretically preferable to retain the full gradient information from $\mathcal{L}_{W}$ while selectively modulating the influence of gradients from $\mathcal{L}_{C}$.
\end{corollary}
Based on Corollary~\ref{cor:motivation}, we propose \mName to prioritize $\mathcal{L}_W$ while controllably incorporating $\mathcal{L}_C$.

\subsection{Stability Analysis}

We analyze the model’s stability under variations in \textit{initialization} and \textit{stochastic mini-batch ordering} induced by random seeds.
Let $\theta$ and $\tilde{\theta}$ represent model parameters trained from two different random initializations $\theta_{0} \neq \tilde{\theta}_{0}$, using the same training dataset $S$ but presented in a different stochastic order (i.e., sampled differently per iteration). The \textbf{stability} of a learning algorithm with respect to the random seed is quantified by the expected divergence:
\begin{equation}
    \label{equ:stability}
    \mathbb{E}_{z \sim \mathcal{D}}\bigl[\, |\mathcal{L}(\theta, z) - \mathcal{L}(\tilde{\theta}, z)| \,\bigr],
\end{equation}
A smaller divergence indicates greater robustness to randomness in initialization and mini-batch ordering. 
We aim to prove that the stability of \mName has a tighter upper bound than SGD.
We consider general update rules of the form $U: \Theta \rightarrow \Theta$ that map a point $\theta \in \Theta$ in the parameter
space to another point $U(\theta)$.
Let $U_{\text{\mName}}$ and $U_{\text{SGD}}$ denote update rules of \mName and SGD respectively, i.e., $U_{\text{SGD}}(\theta)=\theta - \eta \frac{1}{m} G$ and $U_{\text{\mName}}(\theta)=\theta - \eta \frac{1}{m}(G_W + \gamma G_C)$. 
In this paper, we compare $U_{\text{\mName}}$ and $U_{\text{SGD}}$ in terms of stability, i.e., $\mathbb{E}_{z \sim \mathcal{D}}[| \mathcal{L}(\theta, z) - \mathcal{L}(\tilde{\theta}, z) |]$ and $\mathbb{E}_{z \sim \mathcal{D}}[| \mathcal{L}(\psi, z) - \mathcal{L}(\tilde{\psi}, z) |]$ where $\theta$ is optimized by using \mName and $\psi$ is updated using SGD.
With a fixed random seed and model architecture, initialization is identical across update rules, i.e., $\theta_0 = \psi_0$ and $\tilde\theta_0 = \tilde\psi_0$.

\begin{theorem}[Tighter stability bound of \mName]
\label{thrm:batch_stability}
    Assume \(\mathcal{L}\) is \(H\)-Lipschitz.
    Let $\mathcal{B}=\{\mathcal{B}_1,\mathcal{B}_2,\dots,\mathcal{B}_T\}$ and $\Tilde{\mathcal{B}}=\{\Tilde{\mathcal{B}}_1,\Tilde{\mathcal{B}}_2,\dots,\Tilde{\mathcal{B}}_T\}$ be two mini-batch sequences to update the learnable parameters.
    Let \(\theta_T\) and \(\tilde{\theta}_T\) denote the parameters obtained after \(T\) steps using \(U_{\text{\mName}}\) on \(\mathcal{B}\) and \(\Tilde{\mathcal{B}}\), respectively, and similarly let \(\psi_T\) and \(\tilde{\psi}_T\) denote those obtained using \(U_{\text{SGD}}\).  
    Initialization is identical across update rules: \(\theta_0 = \psi_0\) and \(\tilde\theta_0 = \tilde\psi_0\). 
    Let $\epsilon$ and $\epsilon'$ denotes the stability bound for \mName and SGD, i.e.,
    $\forall z \sim \mathcal{D}$,
    $\epsilon \geq \mathbb{E}[|\mathcal{L}(\theta_T, z) - \mathcal{L}(\tilde\theta_T, z)|] ~\text{and}~
    \epsilon' \geq \mathbb{E}[|\mathcal{L}(\psi_T, z) - \mathcal{L}(\tilde\psi_T, z)|]$.
    Then the stability bound of \mName is tighter than that of SGD, i.e.,
    $\epsilon \leq \epsilon'$.
\end{theorem}

Based on the proof of Theorem~\ref{thrm:batch_stability} in Appendix~\ref{prf:batch_stb}, divergence is bounded by a function of the Lipschitz constant \( H \), learning rate \( \eta \), and the total training iteration $T$. While a smaller \( \eta \) tightens the stability bound, it also slows convergence; a smaller $T$ usually results in poor predictive performance due to under-training. 
Our method \mName improves stability without sacrificing convergence speed and accuracy.

\subsection{Convergence Analysis}\label{subsec:convergence}

\begin{theorem}[Convergence of \mName] \label{thrm:convergence_lw}
    Assume that $\mathcal{B}_W$ and $\mathcal{B}_C$ are not empty and $\mathcal{L}_W$ is $M$-smooth. Let $\{\theta_t\}$ be generated by \mName update rule $\theta_{t+1} = \theta_t - \eta\frac{1}{m}(G_W + \gamma_t G_C)$ with $\eta > 0$. If \textbf{(1)} $0 < \gamma_t \neq \gamma_{t+1} \leq \tau$; \textbf{(2)} $\|G_W\|^2 \geq \tau^2\|G_C\|^2$ for all $t$; and \textbf{(3)} $0 < \eta < \frac{m}{M}$, then $\mathcal{L}_W$ will strictly decrease every 2 consecutive iterations under the new update rule or $\mathcal{L}$ and $\mathcal{L}_W$ will converge to stationary points.
\end{theorem}

Note that we schedule $\gamma$ dynamically hence $\gamma_t \neq \gamma_{t+1}$ moreover $\|G_W\|^2 \geq \tau^2\|G_C\|^2$ is a mild condition, since theoretically $\tau$ can be taken arbitrarily small. In our paper, we conduct sensitivity analysis on $\tau$ and discuss the results in Section~\ref{subsec:sensitivity}. 
According to Theorem~\ref{thrm:convergence_lw}, \mName provides a provable escape mechanism from the collapsed state, guaranteeing progress where standard GD would stall.

\textbf{Remark:} When $\mathcal{B}_W$ or $\mathcal{B}_C$ is empty, \mName is equivalent to GD and converges at a rate of $O(1/T)$. 
From the proof of Theorem~\ref{prf:convergence_lw}, \mName stops updating only when $\|G_W\| = \|G_C\| = 0$, indicating convergence to a stationary point. In contrast, GD stops when $\|G_W + G_C\| = 0$ indicating either a convergence or a gradient cancellation ($\|G_W\| > 0$ and $\|G_C\| > 0$) which makes GD learn ineffectively.

%% file: sections/experiment.tex
\input{tables/textual_task}

\section{Experiments}\label{sec:experiment}
We conducted our experiments on various tasks including sequence and image classifications to show the effectiveness of our approach. Additional results on multi-task learning and more analysis can be found in Appendix~\ref{apx:additional_results}. 
Key experimental settings are described below. More details, including but not limited to dataset statistics, training settings, and baseline implementation, can be found in Appendix~\ref{apdix:exp_setup}. 

\subsection{Implementation}\label{subsec:implementation}
\textbf{Dynamic Scaler $\gamma$.} We implemented a linear scheduler:
\begin{equation}
\gamma_t = \zeta(t, T, \tau) = \frac{t}{T} \tau,    
\end{equation}
with $\tau = 1$. In the sensitivity analysis (Section~\ref{subsec:sensitivity}), we compared this scheduler with two alternatives: a descending linear function $\gamma_t = \tau \left(1 - \frac{t}{T}\right)$ and a cosine scheduler $\gamma_t = \frac{\tau}{2} \left[1 - \cos\left(\pi \frac{t}{T}\right)\right]$, along with various static $\gamma$ values and different choices of $\tau$.

\textbf{Metrics.} \textit{Accuracy (ACC)} and \textit{standard deviation (STD)} are used as evaluation metrics, with the mean accuracy and variance derived from 10 arbitrary random seeds. 

\subsection{Sequence Classification}\label{subsec:sequence_classification}
\textbf{Datasets and Models.} 
We conducted experiments on six tasks (MultiRC, COPA, RTE and BoolQ from SuperGLUE~\cite{wang2019superglue}; and MRPC and CoLA from GLUE~\cite{wang2018glue}) using encoder-only \texttt{Roberta-large}~\cite{liu2019roberta} and decoder-only \texttt{Llama-3.2-1B}~\cite{touvron2023llama}, both show high fine-tuning variance on GLUE and SuperGLUE benchmarks in our preliminary experiment and previous work~\cite{mosbach2020stability,wang-etal-2023-two}.

\textbf{Baselines.} 
We compare \mName against (1) fully fine-tuning (\textbf{FFT}) using the training guidelines from~\cite{mosbach2020stability, zhang2020revisiting};
% which adopt small learning rate and optimizer with bias correction to reduce the number of failed runs; 
(2) FFT with focal loss~\cite{lin2017focal} (\textbf{FocalLoss});
\textit{noise injection approaches}: (3) \textbf{LNSR}~\cite{hua2021noise} and (4) \textbf{NoisyTune}~\cite{wu2022noisytune}; \textit{gradient conflict resolver}: (5) \textbf{PCGrad}~\cite{yu2020gradient};
% proposed to resolve the gradient conflicts between tasks in multitask learning; 
\textit{ensemble methods}: (6) bagging ensemble of $N$ single learners based on majority voting (\textbf{ENS ($\times N$)}); (7) stochastic weight averaging (\textbf{SWA})~\cite{izmailov2018averaging}.

\textbf{Overall Performance.}
In Table~\ref{table:nlp_cv}, we report ACC and STD across 10 random seeds for 6 NLP tasks. Compared to single-learner baselines, \mName consistently achieves substantially higher ACC and lower STD across task-model pairs. 
Using RoBERTa-large, \mName achieves an average +7.12\% improvement in ACC (78.66\% vs. 85.78\%) and an average -9.71 reduction in STD (11.13 vs. 1.42) over FFT. Except FocalLoss which we will discuss separately, all other baselines underperform FFT in average ACC while keep similar high STD.
For instance, on MultiRC, \mName achieved 84.01\% ACC with 1.31 STD, significantly outperforming FFT (74.05\% ACC with 14.50 STD), while other baselines (PCGrad 62.03\% $\pm$ 8.70, LNSR 66.07\% $\pm$ 11.46, and NoisyTune 63.62\% $\pm$ 10.72) failed to learn an effective discriminator while keeping similar high variance. Similar trends were observed in RTE and BoolQ by \texttt{Llama-3.2-1B}.

The noise injection approaches, LNSR and NoisyTune, were not able to mitigate the seed-induced instability while also hurting accuracy compared to FFT.
The gap was most evident on MultiRC, RTE, and BoolQ with \texttt{RoBERTa-large}, and RTE, BoolQ, MRPC, and CoLA with \texttt{Llama-3.2-1B}.
One possible reason is that noise injection adds randomness everywhere, which fails to stabilize the sensitive parameter directions that cause divergence and may even increase it by creating more varied starting points.  
The added noise may destroy some of the transferrable knowledge leading to reduced accuracy. 
Instead, our method \mName ensures escape from the collapsed state regardless of the starting points, leading to consistently improved accuracy and stability. 

PCGrad was proposed to solve gradient conflict in multitask learning which we adapted to our problem. However, there are still collapsed states observed despite its lower accuracy compared to FFT in most tasks (see Figure~\ref{fig:roberta_appex_10_runs},\ref{fig:llama_appex_10_runs},\ref{fig:vit_appex_10_runs}). Specifically, by \texttt{RoBERTa-large}, the performance got reduced significantly on MultiRC (62.03\%) and COPA (69.06\%) compared with FFT (74.05\% and 74.70\%). Although the ACC got improved for the other four tasks, the STD still kept high. By \texttt{Llama-3.2-1B}, the ACC of all tasks degraded using PCGrad with increased STD, indicating that PCGrad failed to improve either accuracy or stability.
The reason could be that PCGrad only acts when gradients have significant negative cosine similarity, but a collapsed state can happen even with orthogonal gradients or small-magnitude gradients.  
In contrast, \mName actively adjusts gradients based on learning outcomes that do not depend on the detection of actual gradient cancellation, making it a more robust and proactive defense against task collapse and seed-induced instability.

FocalLoss performs well using \texttt{RoBERTa-large} (81\% ACC 4.53 STD on average) but also shows high instability on COPA (7.59 STD) and MRPC (6.5 STD). However, the best performance across 10 random runs is relatively low compared to other baselines (see Figure~\ref{fig:roberta_appex_10_runs}).
It also shows universal training failures across diverse benchmarks when using \texttt{Llama-3.2-1B} (e.g., BoolQ, MRPC, and CoLA).  
This happens because FocalLoss reshapes the gradient field via confidence-dependent curvature (Section~\ref{appex_subsec:setting}), effectively suppressing gradients from high-confidence examples. As training progresses and most examples become high-confidence, this reduces the effective batch size and leads to unstable updates. A similar effect is observed for \mName with very small $\tau \approx 0$. See Section~\ref{subsec:sensitivity} for more discussion. 
% In contrast, \mName directly addresses gradient cancellation by scaling the magnitude of the gradients without distorting the loss landscape. 
Furthermore, FocalLoss relies on a task-sensitive hyperparameter which we have performed exhaustive search to report the best performance in Table~\ref{table:nlp_cv}. Its inconsistent performance across tasks suggest that tuning is non-trivial. \mName introduces a task-agnostic hyperparameter $\tau$ that generalizes well across all task-model pairs in our experiments.

\textbf{Ensemble Performance.}  
Despite their high computational cost, ensembling techniques are considered as the most effective methods to mitigate random failures in the deep learning domain as they can improve the average accuracy and consistency with theoretical guarantee \cite{wang2020wisdom,wang-etal-2023-two, summers2021nondeterminism, nishida-etal-2025-instability}. To compare our method with ensemble, we applied simple bagging ensemble of FFT in size 3, 5, 7, and 9, as well as an ensemble of \mName in size 9, and show our results in Table~\ref{table:nlp_cv}. 
We observe that it exhibits increasing improvements in accuracy and variance as the ensemble size grows. However, with \texttt{RoBERTa-large} a single \mName consistently outperforms most ensemble sizes across 6 tasks, achieving an average accuracy of 85.78\% with 1.42 STD, surpassing even the largest ensemble ENS ($\times9$) at 85.26\% with 0.6 STD. With \texttt{Llama-3.2-1B}, \mName also outperforms the strong ensemble configuration ENS ($\times7$) in terms of average accuracy.
In addition, the ensemble of \mName performs the best across various tasks and models compared to the ensemble of FFT. 

\citet{wang-etal-2023-two} shows that the variance of an LPM can be decomposed into optimization variance and sampling variance. In over-parameterized settings, optimization variance vanishes as the ensemble size increases while sampling variance remains unchanged, implying an upper bound on variance reduction that strongly depends on ensemble size. This is consistent with our empirical results.  
For instance, with \texttt{RoBERTa-large}, we observe no failed runs only once the ensemble size exceeds task-specific thresholds (e.g., 5 for MultiRC, 3 for MRPC, and 7 for CoLA), with similar behavior observed for \texttt{Llama-3.2-1B}. This threshold is the minimum ensemble size needed to avoid failed runs and is related to the failure rate of individual components on each task. 
% For instance, in Figure~\ref{fig:collapsed_proportion}, there are 4 out of 10 runs failed for MultiRC and we observe no failure runs until ENS ($\times7$). 
However, because ensemble components are sampled randomly, there is no deterministic rule to guarantee failure-free ensembles purely based on size. The threshold should therefore be interpreted as an empirical estimate rather than a strict guarantee.
% By contrast, \mName introduces only one task-agnostic hyperparameter $\tau$ that does not act as a discrete success–failure switch, unlike ensemble size. 
In addition, an ensemble classifier may not necessarily be better than the performance of the best component (Figure~\ref{fig:ens_appendix}).
The "best" FFT component can itself be a failure run selected by random chance, which caps the ensemble's potential, and if all components fail, the ensemble fails as well.  
% In contrast, Theorem~\ref{thrm:convergence_lw} guarantees that our method \mName converges to a robust solution independent of initial component quality. 
In contrast, \mName as a single learner is free of these issues. 
More importantly, compared to an ensemble of size $N$, our method requires only $1/N$ training time of the ensemble. More computational cost analysis is discussed in Section~\ref{subsec:computational_cost}. 

Beyong bagging ensemble, we evaluated SWA which averages nearby checkpoints within a local optimization basin. SWA fails to stabilize performance when the model converges to a collapsed state, as nearby checkpoints remain trapped in the same suboptimal solution. Empirically, SWA fails to mitigate instability in several NLP tasks, such as exhibiting high variances of 11.27 on MultiRC, 13.61 on RTE, 9.61 on BoolQ using \texttt{RoBERTa-large}.

\subsection{Image Classification}

\input{tables/visual_task}

To assess \mName for image classification, we design our experiments with the following configurations and provide justification of our settings in Appendix Section~\ref{apdix:image_classification}.

\textbf{Datasets and Models and Baselines.} Following~\cite{cao2019learning}, we created imbalanced versions of CIFAR-10 and CIFAR-100 by reducing training samples per class while preserving the original validation sets. We consider two imbalance ratios (100:1 and 50:1) and two imbalance patterns: long-tailed (exponential decay)~\cite{cui2019class} and step imbalance~\cite{azizzadenesheli2019regularized}, resulting in four distinct datasets per benchmark.
We used Vision Transformer-based model \texttt{ViT-base} ~\cite{dosovitskiy2020image} as the pretrained model and adapted sequence classification baselines for image classification tasks. 

\textbf{Overall Performance.} Table~\ref{table:vision_result} presents single-learner results across eight datasets. On the more challenging 100:1 imbalance ratio, PCGrad fails to improve model performance in terms of both accuracy and variance, FocalLoss exhibits substantially high variance on CIFAR-10:step (14.22 STD) and CIFAR-100:long-tailed (10.70 STD), while LNSR and NoisyTune yield inconsistent results that improving some tasks but degrading others. In contrast, \mName consistently enhances both accuracy and stability over FFT in a large margin, outperforming all baselines. 
With the easier 50:1 imbalance, FFT already yields strong accuracy and low STD. Most baselines still provide improvements, except NoisyTune and FocalLoss on CIFAR-10:long-tailed which underperform FFT. \mName achieves the best or the second best performance in most tasks. 
Overall, existing baselines show inconsistent gains over FFT, whereas \mName delivers consistent improvements across tasks, achieving superior accuracy and stability.

\subsection{Computational Cost}
\label{subsec:computational_cost}

Figure~\ref{fig:cost} compares accuracy, standard deviation, and training time across all methods using \texttt{RoBERTa-large} on MultiRC. We observe that although baseline methods can mitigate instability for some tasks, they come with significant trade-offs in efficiency, performance, or both.
Single-learner baselines like LNSR, NoisyTune, and FocalLoss require no additional time or storage compared to FFT, yet they underperform in both accuracy and variance, especially using large-scale \texttt{Llama-3.2-1B}. 
PCGrad requires 2 backward passes to compute the gradient of correctly and wrongly classified groups separately hence the training time is nearly doubled.
Bagging ensemble achieves both better accuracy and variance at the expense of linearly increasing computational cost and storage space with ensemble size, while SWA requires extra storage space to save nearby checkpoints yet delivers inconsistent performance gains compared to FFT.  
For some tasks, such as RTE and MRPC, \mName outperforms even the largest ensemble (ENS ($\times9$)) in accuracy while matching its low standard deviation, yet requires no extra time or storage, achieving the best overall performance, similar trends are observed across other tasks.

\begin{figure}[t]
    \centering
    \begin{adjustbox}{max width=\linewidth}
    \includegraphics[width=\linewidth]{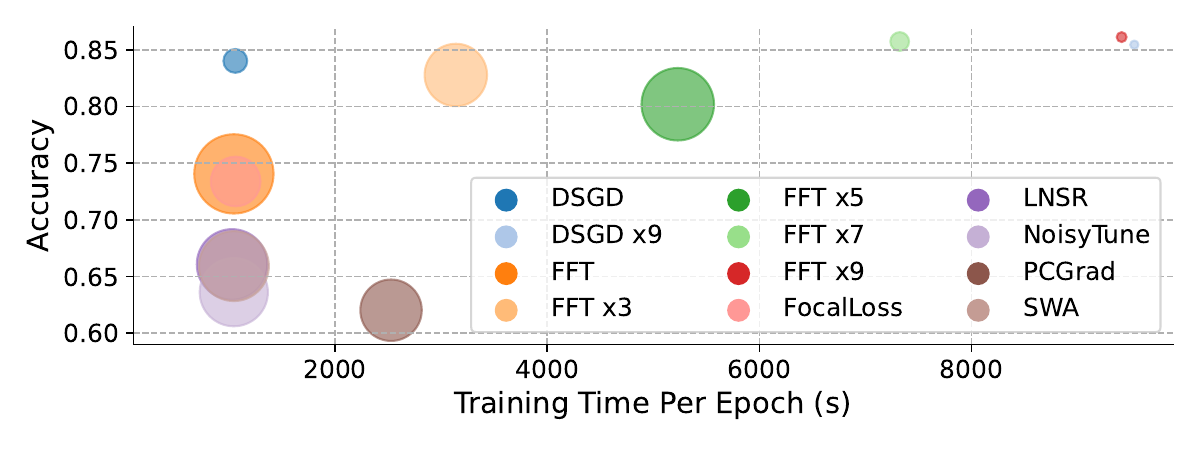}
    \end{adjustbox}
    \caption{Training time per epoch of each methods on MultiRC. Size of the dot denotes standard deviation. 
    % \textcolor{red}{change color to magneta}
    }
    \label{fig:cost}
\end{figure}

\subsection{Sensitivity Analysis and Discussion}
\label{subsec:sensitivity}

\begin{figure*}[t]
    \centering
    \begin{subfigure}{0.245\linewidth}
        \centering
        \includegraphics[width=\linewidth]{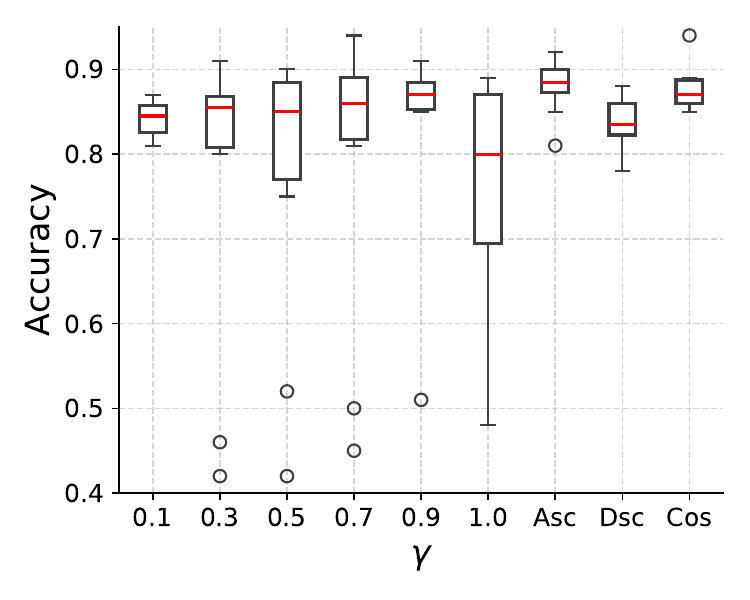}
        \caption{Various scaler $\gamma$ functions.}
        \label{fig:scaler_type}
    \end{subfigure}
    \centering
    \begin{subfigure}{0.245\linewidth}
        \centering
        \includegraphics[width=\linewidth]{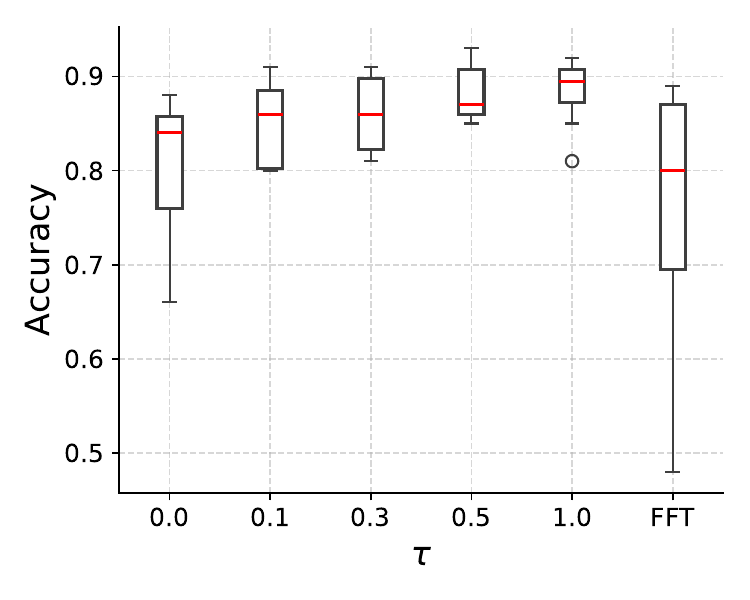}
        \caption{Various $\tau$ settings.}
        \label{fig:tau_scaler}
    \end{subfigure}
    \centering
    \begin{subfigure}{0.245\linewidth}
        \centering
        \includegraphics[width=\linewidth]{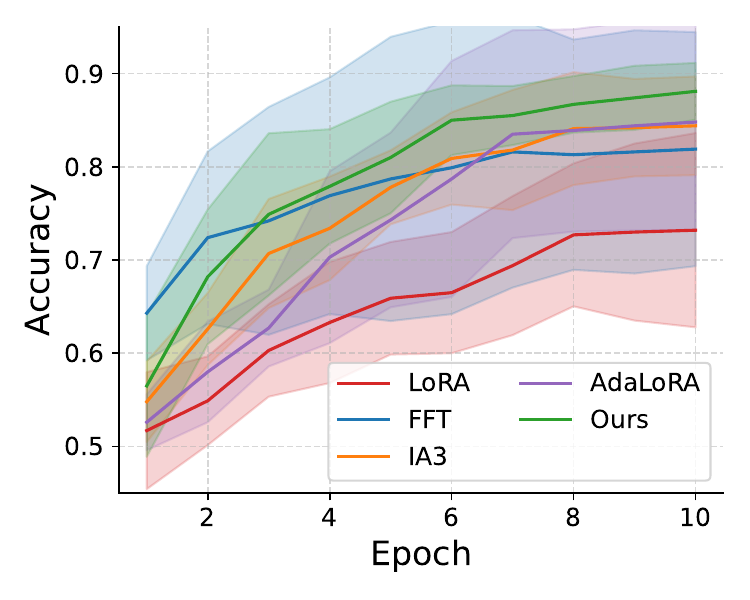}
        \caption{PEFT.}
        \label{fig:peft_comparison}
    \end{subfigure}
    \centering
    \begin{subfigure}{0.245\linewidth}
        \centering
        \includegraphics[width=\linewidth]{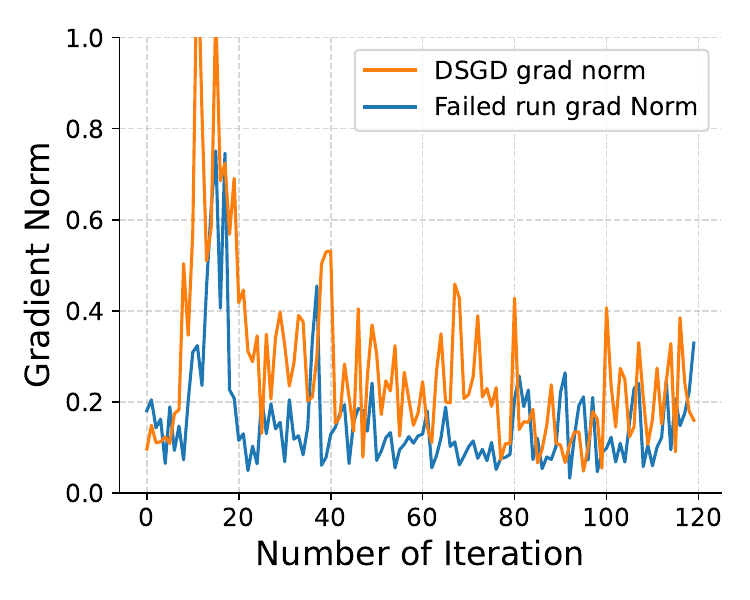}
        \caption{Gradient comparison.}
        \label{fig:gradient_comparison}
    \end{subfigure}
    \caption{Sensitivity analysis conducted using \texttt{RoBERTa-large} on COPA dataset. }
\end{figure*}

In this section we analyze multiple aspects of \mName using \texttt{RoBERTa-large} on COPA.

\textbf{Scaler $\gamma$.}
The superior of dynamic scaler is theoretically proved in Theorem~\ref{thrm:convergence_lw} and ~\ref{thrm:batch_stability}. In our experiment, a linearly ascending scaler consistently outperforms existing methods in both accuracy and stability. To further explore the influence of the scaling strategy, we compare static scalers $\gamma \in \{0.1, 0.3, 0.5, 0.7, 0.9\}$ with two dynamic alternatives: a linear descending function and a cosine scheduler. The results shown in Figure~\ref{fig:scaler_type} show that compared to dynamic scalers, using a static one exhibits larger performance variations and we can observe collapsed states for $\gamma$ above 0.1. While $\gamma=0.1$ avoids collapse, it achieves lower average accuracy. 
On the other hand, dynamic scalers show consistently improved accuracy and stability over FFT, i.e., no collapsed states. Among them, the linearly ascending scaler and cosine scaler perform similarly best. These findings align with the experimental results in Section~\ref{sec:experiment} and the theoretical analysis in Section~\ref{sec:theory}. 

\textbf{Dynamic Scaler Upper Bound $\tau$.}
The hyperparameter $\tau$, introduced in Section~\ref{sec:solution}, governs the maximum scaling amplitude; lower values produce a more aggressive downscaling effect.
To investigate the impact of $\tau$, we vary $\tau \in \{0, 0.1, 0.3, 0.5, 1\}$. As shown in Figure~\ref{fig:tau_scaler}, larger $\tau$ yields better mean accuracy and lower variance. With any $\tau \leq 1$, the performance is consistently higher and more stable than FFT ($\gamma=1$).
This occurs because larger $\tau$ incorporates more signals from correctly classified examples, analogous to increasing the effective batch size, thereby boosting performance and updating stability.
Hence, we set $\tau=1$ in all experiments.

\textbf{Practical Guidelines.} FocalLoss and \mName with configurations like descending linear $\gamma$ or small $\tau$ (0 or 0.1) lead to lower accuracy and higher variance compared to \mName with an ascending linear or cosine $\gamma$ with $\tau=1$. As the model gains confidence, the former settings apply near-zero scaling to "easy" examples, drastically reducing the effective batch size. This results in unstable updates and degraded performance. Conversely, the later settings retain more gradient signal throughout training, ensuring more stable and effective optimization.

\textbf{Size of Trainable Parameters.}
To investigate whether fine-tuning a smaller set of trainable parameters can improve model stability, in Figure~\ref{fig:peft_comparison}, we compare \mName against several parameter-efficient fine-tuning (PEFT) methods~\cite{han2024parameter} including LoRA~\cite{hu2022lora}, AdaLoRA~\cite{zhang2023adalora}, and IA3~\cite{liu2022few}, which are designed to improve fine-tuning efficiency by updating only a subset of model parameters.
Compared to FFT ($\sim$350M learnable parameters), all PEFT methods LoRA ($\sim$787.5K), AdaLoRA  ($\sim$6.2M), IA3 ($\sim$173K), and our method \mName ($\sim$350M) initially exhibit lower accuracies, which gradually rise at different rates over epochs. All methods except LoRA surpass FFT after epoch 7. LoRA exhibits the weakest performance, with accuracy significantly lower than FFT. While the variance of FFT and PEFT methods grows with training time, PEFT approaches maintain relatively lower variance than FFT throughout. Nevertheless, collapsed states and degenerate solutions still occur, suggesting that reducing the number of trainable parameters can mitigate instability at beginning but its effect diminishes as training progresses.
In contrast, \mName exceeds FFT after the third epoch and maintains the best accuracy throughout training while continuously decreasing the variance. These findings align with Theorems \ref{thrm:convergence_lw} and \ref{thrm:batch_stability}, confirming that \mName ensures superior convergence and stability.

\textbf{Mitigation of Gradient Cancellation.}
As discussed in Section~\ref{sec:problem} and shown in Figures~\ref{fig:cosine_conflict} and \ref{fig:cosine_conflict_suc}, failure runs arise from gradient conflicts between correctly and incorrectly classified examples within a batch. \mName alleviates this issue by dynamically downscaling the gradients of correctly classified examples. As shown in Figure~\ref{fig:gradient_comparison}, the resulting gradient norm (orange curve) increases substantially over training compared to the original failed run (blue curve), indicating that \mName effectively mitigates gradient cancellation.
Figure~\ref{fig_apx:failed_suc_comparison} shows more examples.

%% file: tables/textual_task.tex
\begin{table*}[t]
\centering
\begin{adjustbox}{max width=\textwidth}
\begin{tabular}{l|cccccc|cc}
\toprule
\texttt{RoBERTa-large}&MultiRC & COPA   & RTE    & BoolQ   & MRPC & CoLA  & ACC($\uparrow$) & STD($\downarrow$)\\
\midrule
FFT                      &\underline{74.05 $\pm$ 14.50}& 74.70 $\pm$ 15.71 & 76.17 $\pm$ 15.44 & 83.20 $\pm$ 7.39  & 84.90 $\pm$ 5.33&78.95 $\pm$ 8.46  & 78.66&11.13\\
FocalLoss& 73.38 $\pm$ 5.73& 80.40 $\pm$ 7.59& 77.29 $\pm$ 1.70& 81.22 $\pm$ 0.46& 86.52 $\pm$ 6.50& \underline{83.82 $\pm$ 5.20}& \underline{81.10}&\underline{4.53}\\
LNSR&66.07 $\pm$ 11.46& \underline{82.00 $\pm$ 12.58}& 77.44 $\pm$ 15.91& 73.41 $\pm$ 11.86& \underline{88.36 $\pm$ 1.55}&80.75 $\pm$ 8.03  & 78.00&10.23\\
NoisyTune&63.62 $\pm$ 10.72& *55.00 $\pm$ 0.00*& 70.65 $\pm$ 15.53& 71.04 $\pm$ 11.46& 86.91 $\pm$ 6.62&80.67 $\pm$ 8.01  & 71.31&8.72\\
PCGrad&62.03 $\pm$ 8.70& 69.60 $\pm$19.11& \underline{81.05 $\pm$ 10.08}& \underline{83.22 $\pm$ 7.40}& 86.67 $\pm$ 6.52&  81.80 $\pm$ 6.84& 77.39&9.77\\
\rowcolor{gray!20} \mName                     &\textbf{84.01 $\pm$ 1.31}& \textbf{88.10 $\pm$ 3.21}& \textbf{83.61 $\pm$ 1.59}& \textbf{85.69 $\pm$ 0.20}& \textbf{88.38 $\pm$ 1.17}&\textbf{84.92 $\pm$ 1.05}& \textbf{85.78}&\textbf{1.42}\\
\midrule
SWA& 65.92 $\pm$ 11.27& \underline{87.70 $\pm$ 4.06}& 78.45 $\pm$ 13.61& 80.36 $\pm$ 9.61& 88.19 $\pm$ 1.21& 82.46 $\pm$ 7.06& 80.51&7.80\\
ENS ($\times$3)&82.77 $\pm$ 8.99&        82.90 $\pm$ 4.72&        79.86 $\pm$ 3.01 &        85.86 $\pm$ 0.35 & 86.74 $\pm$ 1.40&78.24 $\pm$ 7.86  & 82.82&4.35\\
ENS ($\times$5)&80.19 $\pm$ 12.12&        84.80 $\pm$ 3.43&        80.51 $\pm$ 2.58&        86.13 $\pm$ 0.14 & 87.25 $\pm$ 0.90&79.76 $\pm$ 7.36  & 82.79&4.71\\
ENS ($\times$7)&\underline{85.72 $\pm$ 0.81}&        85.10 $\pm$ 2.13&        79.89 $\pm$ 1.92&        86.20 $\pm$ 0.15 & 87.82 $\pm$ 0.50&80.77 $\pm$ 6.17  & 84.85&1.94\\
ENS ($\times$9)&\textbf{86.11 $\pm$ 0.23}&        {86.40 $\pm$ 1.17}&        \underline{81.16 $\pm$ 1.06}&        \underline{86.31 $\pm$ 0.08}& \underline{88.04 $\pm$ 0.32}&\underline{83.54 $\pm$ 0.75}& \underline{85.26}&\underline{0.60}\\
\rowcolor{gray!20} \mName ($\times$9)&85.44 $\pm$ 0.16& \textbf{90.80 $\pm$ 0.79}& \textbf{84.77 $\pm$ 0.41}& \textbf{86.57 $\pm$ 0.08}& \textbf{88.53 $\pm$ 0.28}&\textbf{86.18 $\pm$ 0.21}& \textbf{87.04}&\textbf{0.32}\\
\midrule
\texttt{Llama-3.2-1B}&MultiRC & COPA   & RTE    & BoolQ   & MRPC & CoLA & ACC($\uparrow$) & STD($\downarrow$)\\
\midrule
FFT                      &62.50 $\pm$ 0.80& 79.40 $\pm$ 6.08& \underline{79.49 $\pm$ 1.83}& 73.98 $\pm$ 3.85 & \underline{86.69 $\pm$ 1.05}&\underline{84.07 $\pm$ 0.83}  & \underline{77.69}&2.41\\
FocalLoss & 56.94 $\pm$ 0.49& 76.00 $\pm$ 7.72& 49.39 $\pm$ 2.97& *62.17 $\pm$ 0.00*& *68.38 $\pm$ 0.00*& *69.13 $\pm$ 0.00*& 63.66&\underline{1.86}\\
LNSR&\underline{65.33 $\pm$ 4.46}& \underline{82.10 $\pm$ 3.14}& 74.90 $\pm$ 2.70& 66.81 $\pm$ 1.24& 82.89 $\pm$ 3.15&78.67 $\pm$ 3.65  & 75.12&3.06\\
NoisyTune&60.85 $\pm$ 0.93 & *55.00 $\pm$ 0.00*& 74.69 $\pm$ 2.92& 67.64 $\pm$ 1.01& 84.76 $\pm$ 2.08&80.91 $\pm$ 0.44& 70.64&\textbf{1.23}\\
PCGrad&60.76 $\pm$ 3.24& 76.40 $\pm$ 2.95& 73.83 $\pm$ 2.95& \underline{67.77 $\pm$ 1.21}& 82.84 $\pm$ 3.39&  80.09 $\pm$ 1.62& 73.61&2.56\\
\rowcolor{gray!20} \mName                   &\textbf{67.17 $\pm$ 4.70}& \textbf{83.30 $\pm$ 2.06}& \textbf{81.55 $\pm$ 1.22}& \textbf{75.43 $\pm$ 3.55}& \textbf{87.52 $\pm$ 0.82}&\textbf{84.75 $\pm$ 0.59}& \textbf{79.95
}&2.16\\
\midrule
SWA& 62.53 $\pm$ 1.01& 55.50 $\pm$ 5.51& 76.17 $\pm$ 1.47& 67.29 $\pm$ 1.64& 82.62 $\pm$ 3.79& 79.95 $\pm$ 2.02& 70.67&2.57\\
ENS ($\times$3)&63.44 $\pm$ 0.42&        81.10 $\pm$ 4.79&        81.91 $\pm$ 1.06&        76.12 $\pm$ 2.90 & 87.65 $\pm$ 0.83&84.94 $\pm$ 0.76  & 79.20&1.79\\
ENS ($\times$5)&64.10 $\pm$ 0.33&        77.60 $\pm$ 9.09&        82.78 $\pm$ 0.83&        76.52 $\pm$ 1.76 & 87.75 $\pm$ 0.66&85.25 $\pm$ 0.34  & 79.00&2.17\\
ENS ($\times$7)&63.79 $\pm$ 0.32& 81.30 $\pm$ 7.18& \underline{83.14 $\pm$ 0.30}&76.57 $\pm$ 0.72 & \underline{88.48 $\pm$ 0.54}&85.11 $\pm$ 0.33  & 79.73&1.56\\
ENS ($\times$9)&\underline{64.20 $\pm$ 0.16}& \underline{83.70 $\pm$ 3.56}& 82.96 $\pm$ 0.28&\underline{76.64 $\pm$ 0.55}& 88.26 $\pm$ 0.69&\underline{85.30 $\pm$ 0.23}& \underline{80.18}&\underline{0.91}\\
\rowcolor{gray!20} \mName ($\times$9)&\textbf{69.11 $\pm$ 0.88}& \textbf{83.90 $\pm$ 3.28}& \textbf{83.54 $\pm$ 0.42}& \textbf{78.09 $\pm$ 0.19}& \textbf{89.26$\pm$ 0.34}&\textbf{85.60 $\pm$ 0.24}& \textbf{81.58}&\textbf{0.89}\\
\bottomrule
\end{tabular}
\end{adjustbox}
\caption{Comparison of \mName with baselines on 6 sequence classification benchmark datasets, evaluated using ACC ($\%$) and STD. Results are averaged over 10 runs, $**$ indicates all collapsed states. The last two columns show the average ACC and STD of 6 NLP tasks. The best performance (based on accuracy) is in \textbf{bold}, and the second best is \underline{underlined}.}
\label{table:nlp_cv}
\end{table*}

%% file: tables/visual_task.tex
\begin{table*}[t]
\centering
\begin{adjustbox}{max width=\linewidth}
\begin{tabular}{l|cc|cc|cc|cc}
\toprule
 \texttt{ViT-base}& \multicolumn{2}{c|}{CIFAR-10: long-tailed}& \multicolumn{2}{c|}{CIFAR-10: step} &  \multicolumn{2}{c|}{CIFAR-100: long-tailed}& \multicolumn{2}{c}{CIFAR-100: step}\\
 \multicolumn{1}{l|}{Ratio} & 100 & 50  & 100 & 50    & 100 & 50  & 100 &50    \\
\midrule
 FFT                                  &     88.47 $\pm$ 1.77&    96.46 $\pm$ 0.44  & 83.27 $\pm$ 2.13 & 96.09 $\pm$ 0.51    & \underline{69.55 $\pm$ 2.71} & 85.53 $\pm$ 0.48  & 56.73 $\pm$ 1.58 &84.30 $\pm$ 0.48   \\
 FocalLoss &\textbf{89.32 $\pm$ 0.68}&95.29 $\pm$ 3.58&70.31 $\pm$ 14.22&\textbf{97.15 $\pm$ 0.57}&64.68 $\pm$ 10.70&87.20 $\pm$ 2.21&57.19 $\pm$ 2.49& 86.60 $\pm$ 4.17\\
 LNSR& 88.05 $\pm$ 1.58&\textbf{97.27 $\pm$ 0.14} & \underline{84.03 $\pm$ 0.98}& 96.90 $\pm$ 0.25& 64.90 $\pm$ 1.29& \textbf{87.90 $\pm$ 0.28} & \underline{57.26 $\pm$ 0.78}&87.72 $\pm$ 0.22  \\
 NoisyTune& 87.88 $\pm$ 1.67&88.94 $\pm$ 9.39  & 68.86 $\pm$ 15.40& 96.80 $\pm$ 0.24   & 68.05 $\pm$ 3.75& \underline{87.89 $\pm$ 0.24}& 57.17 $\pm$ 1.11&\underline{87.89 $\pm$ 0.24}\\
 PCGrad& 86.63 $\pm$ 2.44& 96.85 $\pm$ 0.19 & 82.05 $\pm$ 3.57& 96.84 $\pm$ 0.15   & 68.13 $\pm$ 1.47& 87.02 $\pm$ 0.37 & 50.01 $\pm$ 1.18&87.24 $\pm$ 0.37  \\
 \rowcolor{gray!20} \mName                                 & \underline{89.23 $\pm$ 0.95}& \underline{97.15 $\pm$ 0.17}& \textbf{84.18 $\pm$ 0.66}& \underline{96.95 $\pm$ 0.42}   & \textbf{70.33 $\pm$ 1.18}& 87.72 $\pm$ 0.31  & \textbf{59.91 $\pm$ 0.44}&\textbf{87.93 $\pm$ 0.37}  \\
\bottomrule
\end{tabular}
\end{adjustbox}
\caption{Comparison with baselines on CIFAR-10 and CIFAR-100 with various imbalance settings.}
\label{table:vision_result}
\end{table*}

%% file: sections/conclusion.tex
\section{Conclusion}
This work investigates the seed-induced instability of fine-tuning LPMs on classification tasks. To mitigate this issue, we propose \mName, a modified gradient descent algorithm that directly downscales gradients from correctly classified training examples to reduce gradient conflicts within mini-batches. Through both theoretical analysis and extensive empirical evaluation, we demonstrate that \mName achieves significantly greater stability and higher accuracy than existing baselines.

\textbf{Limitation}: Our approach focuses on correcting the cumulative effects of seed-induced randomness during the optimization process, including not only parameter initialization and data ordering, but also dropout masks, optimizer state initialization, and hardware-level non-determinism which collectively contribute to run-to-run variance. However, a limitation of \mName is that it is designed specifically for gradient-based fine-tuning. It does not address other sources of instability, such as prompt formatting in prompt tuning~\cite{he2024does} and in-context example selection in in-context learning~\cite{gupta2023coverage}.

%% file: sections/appendix.tex
% \section{Related Work}
\input{sections/relatedwork}

\section{Preliminaries}

% \begin{definition}
% \label{def:convergence}
%     \textbf{Convergence:} The state during training when model parameters stabilize, with the loss function exhibiting negligible further decrease. Convergence may result in either a high-performing solution, a degenerate solution, or a collapsed state, depending on the optimization trajectory and loss landscape.
% \end{definition}

\begin{definition}
    \label{def:L-Lipschitz}
    (M-Lipschitz function). A function $f$ is M-Lipschitz if $\forall \theta \in \Theta$ and $\forall z \in \mathcal{D}$: $\| \nabla f(\theta, z) \| \leq M$ which also implies that: $|f(\theta, z) - f(\theta', z)| \leq M\| \theta - \theta' \|$.
\end{definition}

\begin{definition}
    \label{def:lipschitz_notation}
    A function $f \in C^{k, \alpha}_L(\mathbb{R}^n)$ i.e. derivatives up to order $k$ exist and are continuous. All $k$-th order derivatives satisfied: 
    $|\nabla^kf(x) - \nabla^kf(y)| \leq L\|x-y\|^\alpha$.
\end{definition}

\begin{definition}
    \label{def:smooth_func}
    A function $f: \mathbb{R}^n \rightarrow \mathbb{R}$ is $L$-smooth $\Leftrightarrow$ $f \in C^{1,1}_L(\mathbb{R}^n)$.
\end{definition}

\section{Lemmas}\label{apx:lemmas}

\begin{lemma} (Descent lemma~\cite{10.5555/2670022}) \label{lem:taylor_exp} 
Assume that $f$ is $M$-smooth i.e. $f \in C^{1,1}_M(\mathbb{R}^n)$, for any $x, y \in \mathbb{R}^n$ we have: 
\begin{equation}
    |f(y)-f(x)- \langle f'(x), y-x\rangle| \leq \frac{M}{2} \|y-x\|^2.
\end{equation}
\end{lemma}

\begin{lemma}
\label{lemma:batch_alphabounded}
    % Based on Lemma 3.3 in~\cite{hardt2016train}, 
    If $\mathcal{L}$ is H-Lipschitz, then \mName and SGD updates are bounded:
    \begin{equation}
    \label{equ:etaH-bound}
        \| \theta - U_{\text{\mName}}(\theta) \| \leq \varrho\eta H~ \text{and}~ 
        \| \psi - U_{\text{SGD}} (\psi) \| \leq \eta H.
    \end{equation}
    where $\varrho = \frac{1}{m}\sum_i v_i$ and $v_i = 1$ when $z_i$ is wrongly classified according to $\theta$ and $v_i = \gamma$ otherwise.
\end{lemma}

% \begin{lemma}
% \label{lemma:recurrence_bound}
%     Assume that $\mathcal{L}$ is H-Lipschitz. Fix an arbitrary sequence of updates $U_{\text{\mName},1}, U_{\text{\mName},2}, \dots, U_{\text{\mName},T}$ on dataset $S$ and another sequence $\tilde{U}_{\text{\mName},1}, \tilde{U}_{\text{\mName},2}, \dots, \tilde{U}_{\text{\mName},T}$ on dataset $\tilde{S}$. Let $\theta_0$ and $\tilde{\theta}_0$ be the starting points in $\Theta$ and $\theta_0 \neq \tilde{\theta}_0$. Define $\Delta_t = \| \theta_t - \tilde{\theta}_t \|$ where $\theta_t, \tilde{\theta}_t$ are defined recursively through:
%     \begin{equation}
%     \label{equ:recur_theta}
%     \begin{split}
%         \theta_{t+1} &= U_{\text{\mName},t}(\theta_t) = \theta_t - \eta_t \gamma_t \nabla \mathcal{L}(\theta_t),\\ \tilde{\theta}_{t+1} &= \tilde{U}_{\text{\mName},t}(\tilde{\theta}_t) = \tilde{\theta}_t - \eta_t \tilde{\gamma}_t\nabla \mathcal{L}(\tilde{\theta}_t),
%     \end{split}
%     \end{equation}
    
%     Then we have the recurrence relation for \mName:
%     \begin{equation}
%     \label{eq:recurrence_relation}
%         \begin{split}
%             \Delta_0 = \|\theta_0 - \tilde\theta_0\|, \quad \Delta_{t+1} \leq \Delta_t + (\gamma_t+\tilde\gamma_t) \eta_t H. \\
%         \end{split}
%     \end{equation}
%     For SGD with $\forall t:\gamma_t = \tilde\gamma_t = 1$: 
%     \begin{equation}
%         \Delta'_0 = \| \psi_0 - \tilde\psi_0 \|, \quad \Delta'_{t+1} \leq \Delta'_t + 2 \eta_t H.
%     \end{equation}
% \end{lemma}

\begin{lemma}
\label{lemma:batch_recurrence_bound}
    Assume that $\mathcal{L}$ is H-Lipschitz. Fix an arbitrary sequence of updates $U_{\text{\mName},1}, U_{\text{\mName},2}, \dots, U_{\text{\mName},T}$ on sampled $m$-size mini-batch sequence $ \mathcal{B}_1, \mathcal{B}_2, \dots, \mathcal{B}_T$ and another sequence $\tilde{U}_{\text{\mName},1}, \tilde{U}_{\text{\mName},2}, \dots, \tilde{U}_{\text{\mName},T}$ on sampled $m$-size mini-batch sequence 
    $\tilde{\mathcal{B}}_1, \tilde{\mathcal{B}}_2, \dots, \tilde{\mathcal{B}}_T$. Let $\theta_0$ and $\tilde{\theta}_0$ be the starting points in $\Theta$ and $\theta_0 \neq \tilde{\theta}_0$. Define $\Delta_t = \| \theta_t - \tilde{\theta}_t \|$ where $\theta_t, \tilde{\theta}_t$ are defined recursively through:
    \begin{equation}
    \label{equ:recur_theta}
    \begin{split}
        \theta_{t+1} &= U_{\text{\mName},t}(\theta_t) = \theta_t - \eta_t \frac{1}{m} \sum_{i = 1}^{|\mathcal{B}_t|} v_i \nabla \mathcal{L}(\theta_t, z_i),\\ \tilde{\theta}_{t+1} &= \tilde{U}_{\text{\mName},t}(\tilde{\theta}_t) = \tilde{\theta}_t - \eta_t \frac{1}{m} \sum_{i =1}^{|\tilde{\mathcal{B}}_t|} \tilde{v}_i \nabla \mathcal{L}(\tilde{\theta}_t, z_i),
    \end{split}
    \end{equation}
    where $v_i = 1$ if $z_i$ is wrongly classified according to $\theta_t$ and $v_i = \gamma_t$ otherwise, similarly with $\Tilde v_i$.
    Then we have the recurrence relation for \mName:
    \begin{equation}
    \label{eq:recurrence_relation}
        \begin{split}
            \Delta_0 = \|\theta_0 - \tilde\theta_0\|, \quad \Delta_{t+1} \leq \Delta_t + (\varrho_t+\tilde\varrho_t) \eta_t H. \\
        \end{split}
    \end{equation}
    where $\varrho_t = \frac{1}{m}\sum_{i=1}^{|\mathcal{B}_t|} v_i$ and $\Tilde \varrho_t = \frac{1}{m}\sum_{i=1}^{|\tilde{\mathcal{B}}_t|} \Tilde v_i$.
    For SGD, $\forall i:v_i = \tilde v_i = 1$: 
    \begin{equation}
        \Delta'_0 = \| \psi_0 - \tilde\psi_0 \|, \quad \Delta'_{t+1} \leq \Delta'_t + 2 \eta_t H.
    \end{equation}
\end{lemma}

In our analysis, when the random seed and model architecture are fixed, the initialization remains identical across different update rules, i.e., $\theta_0 = \psi_0$ and $\tilde\theta_0 = \tilde\psi_0$.

\section{Proofs}\label{apx:prfs}

\subsection{Proof of Lemma~\ref{lem:property_lw}}
\label{lw_proof}

\begin{proof}
For a cross-entropy loss with softmax outputs, the gradients with respect to 
$\Theta$ and $b$ are
\[
\frac{\partial \mathcal{L}}{\partial \omega}
    = \sum_i (p_i - y_i) x_i^\top, 
    \qquad
\frac{\partial \mathcal{L}}{\partial b}
    = \sum_i (p_i - y_i).
\]
At a stationary point of $\mathcal{L}$, a batch size $m=1$ gives us:
\[
\nabla_{\theta}\mathcal{L} = 0 
\;\;\Longleftrightarrow\;\;
\forall i,\;\; p_i = y_i.
\]
Now consider $\mathcal{L}_W$, the loss restricted to misclassified samples.  
If $\mathcal{L}_W$ has converged, then every misclassified example 
$(x_i, y_i) \in W$ satisfies
\[
(p_i - y_i) = 0 \quad \Longrightarrow \quad p_i = y_i,
\]
which contradicts the definition of misclassification unless the example is now 
correctly classified. Thus all samples previously in $W$ must satisfy  
$\arg\max(p_i) = \arg\max(y_i)$ at convergence.

For all correctly classified samples $(x_i, y_i) \in C$, we already have  
$\arg\max(p_i) = \arg\max(y_i)$ by definition.  
 
Therefore, at convergence of $\mathcal{L}_W$, every training example is correctly 
classified, which implies zero training error (i.e., 100\% training accuracy), 
though the training loss itself may not be zero.
\end{proof}

\subsection{Proof of Lemma~\ref{lem:property_lc}}
\label{lc_proof}

\begin{proof}
If $\mathcal{L}_C$ has converged, then for all correctly classified samples
\[
(x_i, y_i) \in C:\quad p_i = y_i.
\]
However, this condition applies only to $C$ and does not constrain the misclassified set $W$.  
If $|W| > 0$, then there exists some $(x_j, y_j) \in W$ such that
\[
\arg\max(p_j) \neq \arg\max(y_j),
\]
implying strictly positive training error.  
Thus, convergence of $\mathcal{L}_C$ does not guarantee zero training error unless $W$ is empty.
\end{proof}

\subsection{Proof of Lemma~\ref{lem:taylor_exp}}

\begin{proof}
    For all $x, y\in \mathbb{R}^n$ we have:
    \begin{equation}
    \begin{split}
        f(y) 
        &= f(x) + \int_0^1\langle f'(x), \tau (y-x), y-x\rangle d\tau \\
        &= f(x) + \langle f'(x), y-x) \rangle + \int_0^1 \langle f'(x + \tau(y-x)) - f'(x), y-x  \rangle d\tau.
    \end{split}
    \end{equation}
    Hence:
    \begin{equation}
    \begin{split}
        |f(y) - f(x) - \langle f'(x), y-x \rangle| 
        &= |\int_0^1 \langle f'(x+\tau(y-x)) - f'(x), y-x \rangle d\tau| \\
        &\leq \int_0^1 |\langle f'(x+\tau(y-x)) - f'(x), y-x \rangle| d\tau \\
        &\leq \int_0^1 \|f'(x+\tau(y-x)) - f'(x)\| \cdot \| y-x\| d\tau \\
        &\leq \int_0^1 \tau L \cdot \| y-x\|^2 d\tau = \frac{M}{2} \|y-x\|^2.
    \end{split}
    \end{equation}
\end{proof}

\subsection{Proof of Lemma~\ref{lemma:batch_alphabounded}}
\label{prf:alphabounded}
\begin{proof}
    Based on Lemma 3.3 in~\cite{hardt2016train}, consider a mini-batch $\mathcal{B}$ ($|\mathcal{B}| = m$), by $H$-Lipschitz condition we have $\|\nabla \mathcal{L}(\theta)\| \leq H, \|\nabla \mathcal{L}(\psi)\| \leq H$ and $\gamma \in (0, \tau]$ and $\tau \leq 1$ hence:
    \begin{equation}
    \begin{split}
        \| \theta - U(\theta) \| 
        = \| \eta \frac{1}{m} \sum_{i=1}^{|\mathcal{B}|} v_i \nabla \mathcal{L}(\theta, z_i) \|
        \leq  \eta \frac{1}{m} \sum_{i=1}^{|\mathcal{B}|} v_i \|\nabla \mathcal{L}(\theta, z_i)\|
        \leq  \eta  \left(\underbrace{ \frac{1}{m} \sum_{i=1}^{|\mathcal{B}|} v_i } _{\varrho \leq 1}\right)  H
        \leq \eta H.
    \end{split}
    \end{equation}
    where $v_i = 1$ when $z_i$ is truly classified and $v_i = \gamma_t$ otherwise. We also have:
    \begin{equation}
    \begin{split}
        \| \psi - \Tilde{U} (\psi) \| 
        = \| \eta \frac{1}{m} \sum_{i=1}^{|\mathcal{B}|} \nabla \mathcal{L}(\psi, z_i) \|
        \leq \eta \frac{1}{m} \sum_{i=1}^{|\mathcal{B}|} \| \nabla \mathcal{L}(\psi, z_i)\|
        \leq  \eta H.
    \end{split}
    \end{equation}
    This completes the proof.
\end{proof}

\subsection{Proof of Lemma~\ref{lemma:batch_recurrence_bound}}
\label{prf:batch_recurrence_bound}
\begin{proof}
    Based on Lemma~\ref{lemma:batch_alphabounded} that $U_t$ and $\tilde U_t$ for \mName are $\varrho_t \eta_t H$ and $\tilde\varrho_t \eta_tH$-bounded, then by the triangle inequality, we have:
    \begin{equation}
    \label{equ:triangle_inequality}
    \begin{split}
        \Delta_{t+1} 
        &= \| U_t(\theta_t) - \tilde U_t(\tilde\theta_t) \|,\\
        &\leq \| \theta_t - \tilde\theta_t \| + \| U_t(\theta_t) - \theta_t -  \tilde U_t(\theta_t) + \tilde\theta_t \| ,\\
        &\leq \Delta_t + \| U_t(\theta_t) - \theta_t\| + \| \tilde U_t(\tilde \theta_t) - \tilde \theta_t \|, \\
        &\leq \Delta_t + (\varrho_t + \tilde\varrho_t)\eta_t H,
    \end{split}
    \end{equation}
    For SGD case: $v_i = 1$ for all cases from this the proof is completed.
\end{proof}

\subsection{Proof of Theorem~\ref{thrm:batch_stability}}
\label{prf:batch_stb}
\begin{proof}
    By the assumption the loss function $\mathcal{L}$ is $H$-Lipschitz for every example $z \sim \mathcal{D}$, we have:
    \begin{equation}
    \label{eq:stab_bound}
        \mathbb{E}[|\mathcal{L}(\theta_T, z) - \mathcal{L}(\tilde\theta_T, z)|] \leq H \mathbb{E} [\|\theta_T - \tilde\theta_T\|] = H\mathbb{E}[\Delta_T],
    \end{equation}
    
    Using Lemma~\ref{lemma:batch_recurrence_bound}, we have $\Delta_{t+1} \leq \Delta_t + (\varrho_t + \Tilde{\varrho}_t)\eta_t H$ we have:
    \begin{equation}
    \label{equ:expected_bound}
    \begin{split}
        \mathbb{E}[\Delta_{t+1}] 
        \leq \mathbb{E}[\Delta_{t}] 
        &+ \eta H \mathbb{E}[(\varrho_t + \Tilde\varrho_t)] 
    \end{split}
    \end{equation}
    where $\eta_t$ and $\upsilon_t$ are functions of training iteration $t$ which are independent of the data sequences.
    Applying this recursively and using Equation~\ref{eq:stab_bound} we obtain:
    \begin{equation}
    \label{equ:gs_bound}
    \begin{split}
        \mathbb{E}[|\mathcal{L}(\theta_T, z) - \mathcal{L}(\tilde\theta_T, z)|] 
        &\leq H \mathbb{E}[\Delta_T], \\
        &\leq H \mathbb{E}[\Delta_{T-1}] + H^2 \mathbb{E}[(\varrho_{T-1} + \Tilde\varrho_{T-1})] \eta_{T-1}, \\
        &\leq H \mathbb{E}[\Delta_{T-2}] + H^2 \sum_{t=T-2}^{T-1} \mathbb{E}[(\varrho_t + \Tilde\varrho_t)] \eta_{t}, \\
        &\leq \dots ,\\
        &\leq H\mathbb{E}[\Delta_0] + H^2 \sum_{t=0}^{T-1} \mathbb{E}[(\varrho_t + \Tilde\varrho_t)] \eta_{t} = \epsilon. \\
    \end{split}
    \end{equation}
    Similarly, for SGD, using Lemma~\ref{lemma:batch_recurrence_bound} we have $\Delta_{t+1} \leq \Delta_t + 2\eta_t H$, we can obtain SGD's stability upper bound as:
    \begin{equation}
    \label{equ:sgd_bound}
    \begin{split}
        \mathbb{E}[|\mathcal{L}(\psi_T, z) - \mathcal{L}(\tilde\psi_T, z)|] 
        &\leq H\mathbb{E}[\Delta'_T],\\
        &\leq H\mathbb{E}[\Delta'_{T-1}] + 2\eta_{T-1}H^2,\\
        &\leq H\mathbb{E}[\Delta'_{T-2}] + 2\sum_{t=T-2}^{T-1}\eta_{t}H^2,\\
        &\leq \dots, \\
        &\leq H\mathbb{E}[\Delta'_0] + 2 H^2 \sum_{t=0}^{T-1} \eta_t = \epsilon'.
    \end{split}
    \end{equation}

    Equations~\ref{equ:gs_bound} and~\ref{equ:sgd_bound} hold for any $z \sim \mathcal{D}$, we analyze a model's stability under different initializations and stochastic mini-batch ordering introduced by different random seeds. In our analysis, given a fixed random seed and a fixed model architecture, the initialization is the same for different update rules, i.e., $\theta_0 = \psi_0$ and $\tilde\theta_0 = \tilde\psi_0$,  which also means $\Delta_0 = \Delta'_0$.
    
    Note that $\varrho_t \in (0, \tau]$ and $\Tilde \varrho_t \in (0, 1]$, then $\mathbb{E}[(\varrho_t + \Tilde\varrho_t)] \leq 2$.
    Thus, the right side of Eq.~\ref{equ:gs_bound} satisfies:
        \begin{equation}
        \label{equ:bound_comparison}
        \underbrace{H\mathbb{E}[\Delta_0] + H^2 \sum_{t=0}^{T-1} \mathbb{E}[(\varrho_t + \Tilde\varrho_t)] \eta_{t}}_{\epsilon}
        \leq
        H\|\theta_0 - \tilde\theta_0\| + 2H^2 \sum_{t=0}^{T-1} \eta_t
        =
        \underbrace{H\|\psi_0 - \tilde\psi_0\| + 2H^2 \sum_{t=0}^{T-1} \eta_t}_{ \epsilon'}.
    \end{equation}
     
    Thus, based on Equation~\ref{equ:gs_bound},~\ref{equ:sgd_bound}, and \ref{equ:bound_comparison}, the stability bound of \mName is lower or equals to that of SGD.
\end{proof}

\subsection{Proof of Theorem~\ref{thrm:convergence_lw}} \label{prf:convergence_lw}

% Assume that $\nabla \mathcal{L}_W$ is $M$-Lipschitz, $\mathcal{L}_W$ will strictly decrease or converge to a stationary point under our update rule: $\theta_{t} = \theta_{t-1} - \eta(G_W + \gamma_{t} G_C)$.

\begin{proof}
We consider $y=\theta_{t+1}$ as the updated parameters at $(t+1)$-th iteration, $x=\theta_{t}$ denotes the previous parameters, $f(\cdot)=\mathcal{L}_W(\cdot)$ as the loss function. Since $\nabla \mathcal{L}_W$ is $M$-Lipschitz, from Lemma~\ref{lem:taylor_exp}, it follows in particular that: 

\begin{equation}
    \mathcal{L}_W(\theta_{t+1}) \leq \mathcal{L}_W(\theta_{t}) +  \langle \nabla \mathcal{L}_W(\theta_{t}), \theta_{t+1} -\theta_{t} \rangle + \frac{M}{2}\|\theta_{t+1} - \theta_{t}\|^2.
\end{equation}

We update the parameters using a modified gradient step: $\theta_{t+1} = \theta_{t} - \eta \frac{1}{m}( \gamma_{t} G_C + G_W)$
where $G = G_C + G_W$ denotes the original gradient.
We have:
\begin{equation}
\begin{split}
    \mathcal{L}_W(\theta_{t+1})
    &\leq \mathcal{L}_W(\theta_{t}) + \langle \nabla \mathcal{L}_W (\theta_{t}), \theta_{t+1} - \theta_{t} \rangle + \frac{M}{2} \| \theta_{t+1} - \theta_{t} \|^2,\\
    &\leq \mathcal{L}_W(\theta_{t}) + \langle G_W, -\eta \frac{1}{m} (G_W + \gamma_{t} G_C) \rangle + \frac{M}{2} \| -\eta \frac{1}{m}(G_W + \gamma_{t} G_C) \|^2, \\
    &\leq \mathcal{L}_W(\theta_{t}) - \eta \frac{1}{m} \langle G_W, G_W + \gamma_{t} G_C \rangle + \frac{M\eta^2}{2m^2} \| G_W + \gamma_{t} G_C \|^2, \\
    &\leq \mathcal{L}_W(\theta_{t}) - \eta \frac{1}{m} ( \|G_W\|^2 + \gamma_{t} G_W^\top G_C ) + \frac{M\eta^2}{2m^2} \| G_W + \gamma_{t} G_C \|^2, \\
    &\leq \mathcal{L}_W(\theta_{t}) - \frac{\eta}{2m}( 2\|G_W\|^2 + 2\gamma_{t} G_W^\top G_C ) + \frac{M\eta^2}{2m^2} \| G_W + \gamma_{t} G_C \|^2.
\end{split}
\end{equation}

We scale $G_C$ down by $\gamma_{t} \in (0, \tau]$ such that $\|G_W\|^2 \geq \tau^2\|G_C\|^2 \geq \gamma_t^2\|G_C\|^2$ then:
\begin{equation}
\label{eq:recurrence_loss}
\begin{split}
    \mathcal{L}_W(\theta_{t+1})
    &\leq \mathcal{L}_W(\theta_{t}) - \frac{\eta}{2m}( \|G_W\|^2 + \gamma_{t}^2\|G_C\|^2 + 2\gamma_{t} G_W^T G_C ) + \frac{M\eta^2}{2m^2} \| G_W + \gamma_{t} G_C \|^2,\\
    &\leq \mathcal{L}_W(\theta_{t}) - \frac{1}{m}\left(\frac{\eta}{2} - \frac{M\eta^2}{2m}\right) \| G_W + \gamma_{t} G_C \|^2,
\end{split}
\end{equation}
hence if $0 < \eta < \frac{m}{M}$ then $\mathcal{L}_W$ strictly decrease unless $\| G_W + \gamma_{t} G_C \|^2 = 0$. 

When $\|G_W + \gamma_{t} G_C\|^2 = 0$ i.e. $G_W +\gamma_{t} G_C = 0$ which indicates that $\theta_{t+1} = \theta_{t} - \eta \frac{1}{m} (G_W + \gamma_{t} G_C) = \theta_{t}$.
Since $\gamma_{t}$ is dynamic, we consider the next iteration:
\[
    \theta_{t+2} = \theta_{t+1} - \eta \frac{1}{m} (G_W + \gamma_{t+1} G_C) = \theta_{t} - \eta \frac{1}{m} (G_W + \gamma_{t+1} G_C)
\]
In iteration $t+1$, the gradients $G_W$ and $G_C$ remain unchanged due to $\theta_{t+1} = \theta_{t}$, thus $G_W +\gamma_{t} G_C = 0$ still holds. Since $\|G_W\|^2 \geq \tau^2\|G_C\|^2 \geq \gamma_{t+1}^2\|G_C\|^2$ we obtain:
\begin{equation}
\label{equ:l_w_strict_decrease}
\begin{split}
    \mathcal{L}_W(\theta_{t+2})
    &\leq \mathcal{L}_W(\theta_{t+1}) + \langle \nabla \mathcal{L}_W (\theta_{t+1}), \theta_{t+2} - \theta_{t+1} \rangle + \frac{M}{2} \| \theta_{t+2} - \theta_{t+1} \|^2,\\
    &\leq \mathcal{L}_W(\theta_{t}) + \langle \nabla \mathcal{L}_W (\theta_{t}), \theta_{t+2} - \theta_{t} \rangle + \frac{M}{2} \| \theta_{t+2} - \theta_{t} \|^2,\\
    &\leq \mathcal{L}_W(\theta_{t}) - \langle G_W,  \eta (G_W + \gamma_{t+1} G_C) \rangle + \frac{M\eta^2}{2m^2} \| G_W + \gamma_{t+1} G_C \|^2,\\
    &\leq \mathcal{L}_W(\theta_{t}) - \frac{\eta}{2m} ( 2\|G_W\|^2 + 2\gamma_{t+1} G_W^\top G_C) + \frac{M\eta^2}{2m^2} \| G_W + \gamma_{t+1} G_C \|^2,\\
    &\leq \mathcal{L}_W(\theta_{t}) - \frac{1}{m}\left(\frac{\eta}{2} - \frac{M\eta^2}{2m}\right) \| G_W + \gamma_{t+1} G_C \|^2.
\end{split}
\end{equation}
There are two cases based on ~\ref{equ:l_w_strict_decrease}: 

\textbf{(1)} If $\| G_W + \gamma_{t+1} G_C \|^2 > 0$, using telescoping sums obtained from~\ref{eq:recurrence_loss} we have: $\mathcal{L}_W(\theta_{T+1}) - \mathcal{L}_W(\theta_0) \leq - \frac{1}{m} \sum_{t=0}^T \left( \frac{\eta}{2} - \frac{M \eta^2}{2m} \right) \| G_{W, t} + \gamma_t G_{C, t} \|^2$. Therefore:
\begin{equation}
    \text{min}_{t \leq T} \frac{1}{m} \|G_{W, t} + \gamma_t G_{C, t}\|^2 \leq \frac{1}{m(T+1)} \sum_{t=0}^T \|G_{W, t} + G_{C, t}\|^2 \leq \frac{2 (\mathcal{L}_W(\theta_0) - \mathcal{L}_W(\theta_{T+1}))}{(\eta - M\eta^2/m) (T+1)},
\end{equation}
hence if $\mathcal{L}_W$ is lower bounded, i.e., $\mathcal{L}^*_W := \text{inf}_{\theta \in \Theta} \mathcal{L}_W > -\infty$ then $\text{min}_{t \leq T}\|G_{W, t} + \gamma_t G_{C, t}\|^2 = O(1/T)$. $\mathcal{L}_W $ decreases at a rate of $O(1/T)$.

\textbf{(2)} If $G_W +\gamma_{t+1} G_C = 0$, since $\gamma_{t+1} \neq \gamma_{t} > 0$ and $G_W +\gamma_{t} G_C = 0$, then $G_W = G_C = 0$. Hence $\mathcal{L_W}$, $\mathcal{L_C}$, and $\mathcal{L}$ converges to a stationary point. 
\end{proof}

\begin{corollary}
    If $G_W = 0$ or $G_C = 0$ DSGD also converges at a rate of $O(1/T)$, the same as SGD due to DSGD now becomes SGD. i.e.
    \begin{itemize}
        \item $G_W = 0$: $U_\mName(\theta) = \theta - \eta \frac{1}{m}\gamma G_C = \theta - \eta \gamma \frac{1}{m}G$ hence \mName becomes SGD with learning rate of $\eta \gamma$.
        \item $G_C = 0$: $U_\mName(\theta) = \theta - \eta \frac{1}{m} G_W = \theta - \eta \frac{1}{m}G$ thus \mName becomes SGD with learning rate of $\eta$.
    \end{itemize}
\end{corollary}

\section{Detailed Experiment Setup}
\label{apdix:exp_setup}
% \textcolor{red}{Todo: need to smooth this section and make sure all tables and figures are well cited and discussed. }
In this section we detail the settings for our experiments.

\subsection{Implementation}
We arbitrarily choose 10 random seeds (42, 52, 62, 72, 82, 92, 102, 112, 122, 132) to obtain the mean and variance performance. All single runs are conducted on 2 NVIDIA-A100-80GB GPUs.

\subsection{Sequence Classification}
We conduct our experiments on 4 tasks on Super-GLUE benchmark including: RTE, COPA, MultiRC and BoolQ data statistics is shown in Table~\ref{table:text_data_stat}.
We finetune pretrained language model with different type of architecture including encoder only: \texttt{RoBERTa-large}, decoder only: \texttt{Llama-3.2-1B}.
To ensure a proper baseline implementation, we refer to the configurations in~\cite{liu2019roberta} and replicate the state-of-the-art (SOTA) scores reported using \texttt{\small RoBERTa-large} with FFT. 
For \texttt{Llama-3.2-1B}, no established reference performance was available in the literature; we therefore fine-tune it using our own settings. 
Although our implementation may not achieve SOTA performance, it establishes a consistent basis for comparison across different methods. 
% Detailed settings can be found in Appendix~\ref{apdix:exp_setup}.

\textbf{Baselines}. We reimplement LNSR~\cite{hua2021noise} by ourself and using our own hyper-parameters to finetune backbone models for NoisyTune~\cite{wu2022noisytune} we use the example code snippets provided in the paper and obtain the reported hyper-parameter settings for GLUE tasks.

\begin{table*}[h]
\centering
\begin{adjustbox}{max width=\textwidth}
\begin{tabular}{lcccccc}
\toprule
& MultiRC & RTE  & COPA & BoolqQ  & CoLA&MRPC\\
\midrule
Number of Classes  & 2       & 2    & 2    & 2       & 2&2\\
Training Samples   & 27243   & 2490 & 400  & 9427    & 8551&3668\\
Validation Samples & 4848    & 277  & 100  & 3270    & 1043&408\\
\bottomrule
\end{tabular}
\end{adjustbox}
\caption{Sequence classification data statistics.}
\label{table:text_data_stat}
\end{table*}

\begin{table*}[h]
\centering
\begin{tabular}{lcccccc}
\toprule
               & MultiRC    & RTE    & COPA    & BoolQ     &CoLA &MRPC\\
\midrule
Batch Size& 5 & 5 & 5 & 5                     & 5&5\\
Gradient Accumulation   & 2 & 2 & 2 & 2                      & 2&2\\
Learning Rate& \multicolumn{6}{c}{5E-6; 1E-5; 2E-5; 3E-5; 1E-4; 3E-4}\\
Learning Rate Scheduler& linear & linear & linear & linear                & linear&linear\\
Max Sequence length& 512        & 512    & 256     & 512        & 256&256\\
Epoch                   & 6 & 10 & 10 & 10                     & 3&3\\
Number of GPUs          & 2 & 2 & 2 & 2                      & 2&2\\
\bottomrule
\end{tabular}
\caption{Roberta and Llama hyper parameters settings.}
\label{table:model_settings}
\end{table*}

\begin{figure*}[t]
  \begin{minipage}{0.5\linewidth}
    \centering
    \begin{tabular}{lccc}
    \toprule
                  & IA3  & AdaLoRA      & LoRA         \\
    \midrule
    Learning rate & 5E-3& 5E-3& 3E-4\\
    % Batch size    & 5    & 5            & 5            \\
    Init r        & -    & 64           & -            \\
    Target r      & -    & 16           & -            \\
    r             & -    & -            & 8, 16        \\
    alpha         & -    & -            & 16           \\
    Target module & -    & query, value & query, value \\
    \bottomrule
    \end{tabular}
    \caption{PEFT settings.}
    \label{table:peft_settings}
  \end{minipage}\hfill
  \begin{minipage}{0.5\linewidth}
    \centering
    \begin{tabular}{lcc}
    \toprule
    & CIFAR10& CIFAR100\\
    \midrule
    Learning Rate& 1E-3& 1E-3\\
    % Learning Rate Scheduler& linear & linear \\
     Batch Size& 128& 128\\
     Epoch                   & 5& 5\\
     Gradient Accumulation   & 2 & 2 \\
     
     Number of GPUs          & 2 & 2 \\
    \bottomrule
    \end{tabular}
    % \end{adjustbox}
    \caption{ViT settings.}
    \label{table:img_settings}
  \end{minipage}
\end{figure*}

\subsection{Settings \& Implementations}
\label{appex_subsec:setting}
We follow~\cite{liu2019roberta, hu2022lora} settings to finetune RoBERTa and Llama on GLUE and SuperGLUE benchmarks. All approaches are optimized using AdamW~\cite{loshchilov2017decoupled} optimizer.

\textbf{Baselines:} 
% \textcolor{red}{list any other settings that you think are necessary to replicate the experiment.}
\begin{itemize}
    \item For \textbf{PCGrad}~\cite{yu2020gradient} we customize it for our task where we consider each group of training example in every iteration as an individual task in multi-task learning context and apply PCGrad learning procedure. The pseudo code of PCGrad is shown in Algorithm~\ref{alg:pcgrad}.
    \item For the \textbf{LNSR}~\cite{hua2021noise} approach, we apply a noise regularizer to the last two layers of the classifier, with $\lambda$ (the weight of the regularization loss) set to 0.2 and 0.1, respectively.
    \item For \textbf{NoisyTune}~\cite{wu2022noisytune}, we add uniform noise $U(a, b)$, where $\lambda$ controls the noise intensity. For the GLUE dataset, we set $\lambda = 0.15$, following~\cite{wu2022noisytune}. For other task we use the same $\lambda$ setting.
    \item \textbf{SWA}~\cite{izmailov2018averaging} maintains an average of model parameters over previous checkpoints. We tune the averaging window size in the range ${3, 5}$ and report the best result for each task.
    \item The settings for the PEFT methods LoRA~\cite{hu2022lora}, AdaLoRA~\cite{zhang2023adalora}, and IA3~\cite{liu2022few} are listed in Table~\ref{table:peft_settings}. All settings follow those used in the respective reference papers.
    \item Although the original implementation of \textbf{Focal Loss}~\cite{lin2017focal} suggests default hyperparameters of $\alpha = 0.25$ and $\gamma = 2$, the modified loss for a single data point $z_i$ is defined as: $\mathcal{L}(z_i) = -\alpha_t(1-p_t)^\gamma \text{log}(p_t)$ where $p_t$ denotes the predicted probability of the ground-truth class. These configurations frequently induced training divergence or collapsed states in our preliminary evaluations. Consequently, we performed an extensive hyperparameter search for $\alpha$ and $\gamma$ tailored to each task and pretrained architecture. The results reported herein represent the optimal performance achieved after this exhaustive tuning process. 
    % set $\alpha = 0.25$ and $\gamma = 2$ for all experiments, following the original paper.
\end{itemize}

\textbf{\mName:}
The pseudo code of \mName is stated in Algorithm~\ref{alg:dsgd}.

\begin{algorithm}[h]
\caption{PCGrad for single task learning}
\label{alg:pcgrad}
\begin{algorithmic}[1]
    \REQUIRE Model parameter $\theta = \theta_0$; training dataset distribution $\mathcal{D}$; mini-batch $\mathcal{B}$ of $m$ examples; loss function $loss(\cdot)$.
    \FOR{i from 1 to T}
        \STATE Sample a mini-batch training examples from $\mathcal{D}$ 
        \STATE Let $W \leftarrow \emptyset$ and $C\leftarrow\emptyset$
        \FOR{j from 1 to B}
            \IF{$f(\theta_i)$ equals $y_i$}
                \STATE $C \leftarrow C + \{z_i\}$
            \ELSE 
                \STATE $W \leftarrow W + \{z_i\}$
            \ENDIF
        \ENDFOR
        \STATE Let $\mathcal{L}_W$ be None, $\mathcal{L_C}$ be None
        \STATE Compute $\mathcal{L}_W$ = $\sum$ $loss$($z_i$) where $i \in W$
        \STATE Compute $\mathcal{L}_C$ = $\sum$ $loss$($z_i$) where $i \in C$
        \IF{$\mathcal{L}_W$ is None and $\mathcal{L}_C$ is not None}
            \STATE $\theta_{i+1} \leftarrow \theta_i - \eta \frac{1}{|C|}\nabla \mathcal{L_C}$
        \ELSIF{$\mathcal{L}_W$ is not None and $\mathcal{L}_C$ is None}
            \STATE $\theta_{i+1} \leftarrow \theta_i - \eta \frac{1}{|W|}\nabla \mathcal{L_W}$
        \ELSE
            \IF{$\mathrm{cos}(\nabla L_W, \nabla L_C) \geq 0$}
                \STATE $\theta_{i+1} \leftarrow \theta_i - \eta \frac{1}{m} \left( \nabla  \mathcal{L_W} + \nabla \mathcal{L_C} \right)$
            \ELSE
                \STATE $\theta_{i+1} \leftarrow \theta_i - \eta \frac{1}{m} \left( \nabla \mathcal{L_W} + \nabla \mathcal{L}_C  - \frac{\nabla \mathcal{L}_W \cdot \nabla \mathcal{L}_C}{\|\nabla \mathcal{L}_C\|} \nabla \mathcal{L}_C - \frac{\nabla \mathcal{L}_W \cdot \nabla \mathcal{L}_C}{\|\nabla \mathcal{L}_W\|} \nabla \mathcal{L}_W\right)$
            \ENDIF
        \ENDIF
    \ENDFOR
    \STATE Return $\theta_T$
\end{algorithmic}
\end{algorithm}

\begin{algorithm}[h]
\caption{\mName}
\label{alg:dsgd}
\begin{algorithmic}[1]
    \REQUIRE Model parameter $\theta = \theta_0$; training dataset distribution $\mathcal{D}$; mini-batch $\mathcal{B}$ of $m$ examples; scaler factor $\tau$; scaler function $\zeta(t, \tau,T)$; loss function $loss$.
    \FOR{t from 1 to T}
        \STATE $\gamma_t = \zeta(t, \tau, T)$
        \STATE Sample a mini-batch training examples $z_i = (x_i,y_i):i \in 1,\dots,m$ from $\mathcal{D}$ 
        \STATE Let $W \leftarrow \emptyset$ and $C \leftarrow \emptyset$
        \FOR{j from 1 to $m$}
            \IF{$f(\theta_i, x_i)$ equals $y_i$}
                \STATE $C \leftarrow C + \{z_i\}$
            \ELSE 
                \STATE $W \leftarrow W + \{z_i\}$
            \ENDIF
        \ENDFOR
        \STATE $\mathcal{L}_W \leftarrow 0$, $\mathcal{L_C} \leftarrow 0$
        \STATE Compute $\mathcal{L}_W$ = $\sum$ $loss$($z_i$) where $i \in W$
        \STATE Compute $\mathcal{L}_C$ = $\sum$ $loss$($z_i$) where $i \in C$
        % \STATE Compute $\mathcal{L} = \mathcal{L}_W + \gamma_t\mathcal{L}_C$
        % \IF {$W = \emptyset$}
        %     \STATE Compute the total loss $\mathcal{L} = \sum_{i \in C}$ $loss(z_i)$
        % \ELSIF{$C = \emptyset$}
        %     \STATE Compute the total loss $\mathcal{L} = \sum_{i \in W}$ $loss(z_i)$
        % \ELSE
        %     \STATE Compute the total loss $\mathcal{L} = \sum_{i \in W}$ $loss(z_i)$ $+ \gamma_t \sum_{i \in C}$ $loss(z_i)$
        % \ENDIF
        % \STATE $\theta_{i+1} \leftarrow \theta_i - \eta \frac{1}{m} \nabla \mathcal{L}$
        \STATE $\theta_{i+1} \leftarrow \theta_i - \eta \frac{1}{m} (\nabla \mathcal{L}_W + \gamma_t \nabla  \mathcal{L}_C)$
    \ENDFOR
    \STATE Return $\theta_T$
\end{algorithmic}
\end{algorithm}

\subsection{Multi-task learning}
For the multi-task learning experiments, we train the model on the joint training data from all tasks and evaluate it on the validation set. The model settings are the same as those used in single-task learning, as shown in Table~\ref{table:model_settings}, and the model is trained for 8 epochs. 

\subsection{Image classification}
\label{apdix:image_classification}

To assess our method for image classification, we first identify tasks affected by seed-induced variance. Initial experiments on balanced CIFAR-10/100 using \texttt{Resnet-34}, \texttt{ResNet50}, \texttt{EfficientNet-b0}, and \texttt{ViT-base} revealed minimal variance across all models. We then created imbalanced versions following ~\cite{cao2019learning}, finding that CNN-based models exhibit significantly lower performance but comparably small variance relative to ViT-based models, which is consistent with prior work on long-tailed recognition. Thus, in our experiment, we conduct experiments on imbalanced datasets using \texttt{ViT-base}. 

\textbf{Datasets.} CIFAR-10 and CIFAR-100 contain $50000$ training points and $10000$ validation images of size $32 \times 32$ with 10 and 100 classes respectively. To create their imbalance version, \citeauthor{cao2019learning} reduce the number of training examples per class and keep the validation set unchanged. There are two types of imbalance: long-tailed~\cite{cui2019class} and step imbalance~\cite{buda2018systematic}

\begin{table*}[h]
\centering
\begin{tabular}{llcc}
\toprule
                               Imbalance type&       Ratio& CIFAR-10    & CIFAR-100   \\
\midrule
\multirow{2}{*}{Long-Tailed}   & 100   & 12406       & 10847       \\
                               & 50    & 36223       & 36029       \\
\multirow{2}{*}{Step}          & 100   & 25250       & 25250       \\
                               & 50    & 37500       & 37500       \\
\midrule
\multicolumn{2}{c}{Validation samples} & \multicolumn{2}{c}{10000} \\
\bottomrule
\end{tabular}
\caption{Image classification data statistics.}
\label{table:img_data_stat}
\end{table*}

We follow~\cite{dosovitskiy2020image} recommendation on finetuning ViT on CIFAR-10 and CIFAR-100 datasets using a learning rate of $1E-3$ and a batch size of $256$ detailed setup can be found in Table~\ref{table:img_settings}.

\section{Additional Results} \label{apx:additional_results}

\subsection{Multi-tasks Learning}
We also conduct experiments in multi-task learning settings, where multiple tasks are learned jointly using \texttt{RoBERTa-large} as a shared backbone model with task-specific classification heads for each individual task. More specifically, the result on three multi-task learning problem are shown in Table~\ref{table:multitask_result}.

\input{tables/joint_task}

\textbf{Results}. From the aggregated results on \texttt{RoBERTa-large} across MRPC, RTE, and CoLA, \textbf{DSGD} clearly achieves the {best overall performance}, with the {highest average accuracy (85.41)} and {by far the lowest standard deviation (1.19)}. This indicates not only strong accuracy but also {excellent stability across runs}, which is particularly important for small and sensitive GLUE tasks like RTE and CoLA.
In contrast, most baselines exhibit either lower mean performance, high variance, or both.

% \subsection{*Will the training size impact our method's improvement?} 

% We examine the variance reduction of \mName on various size dataset, as shown in Figure~\ref{fig:training_std}, regardless the training size, the variance of \mName is significantly better than FFT. The STD using traditional FFT shows a low correlation to the training size, despite the size of the training set .

% \begin{figure}[h]
%     \centering
%     \includegraphics[width=0.8\linewidth]{figures/training_std.pdf}
%     \caption{Standard deviation improvements across tasks}
%     \label{fig:training_std}
% \end{figure}

% \begin{figure*}[h]
%   \begin{minipage}{0.48\linewidth}
%     \centering
%     \includegraphics[width=\linewidth]{figures/cost_llama.pdf}
%     \caption{Computational cost using \texttt{Llama-3.2-1B} as backbone model.}
%     \label{fig:llama_cost}
%   \end{minipage}\hfill
%   \begin{minipage}{0.48\linewidth}
%     \centering
%     \includegraphics[width=\linewidth]{figures/cost_roberta.pdf}
%     \caption{Computational cost using \texttt{RoBERTa-large} as backbone model. \textcolor{red}{update to include all baselines and change computational time.}}
%     \label{fig:roberta_cost}
%   \end{minipage}
% \end{figure*}

\begin{figure*}[t]
    \centering
    \begin{subfigure}{0.3\linewidth}
        \centering
        \includegraphics[width=\linewidth]{figures/fai.pdf}
        \caption{Original failed run.}
        \label{fig_apx:cosine_conflict}
    \end{subfigure}
    \hspace{10pt}
    \begin{subfigure}{0.3\linewidth}
        \centering
        \includegraphics[width=\linewidth]{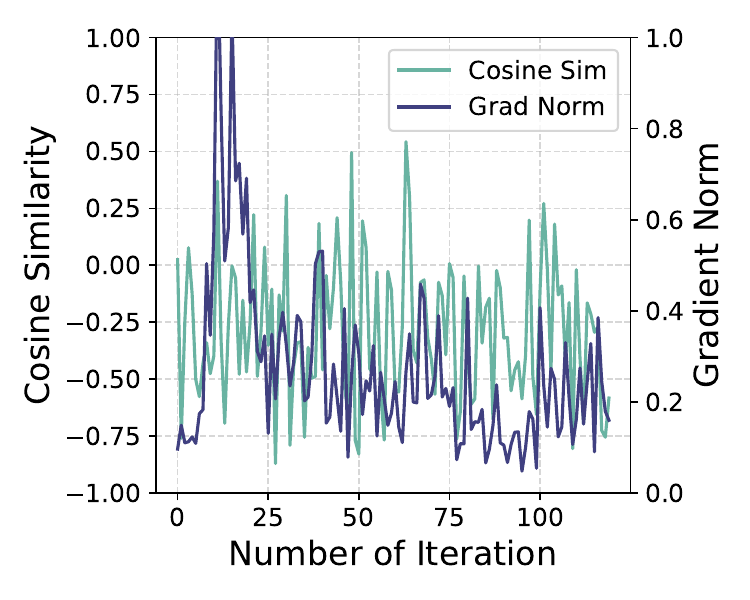}
        \caption{Stable run corrected by \mName.}
        \label{fig_apx:dsgd_grad}
    \end{subfigure}
    \hspace{10pt}
    \begin{subfigure}{0.3\linewidth}
        \centering
        \includegraphics[width=\linewidth]{figures/gradnorm_comparison.pdf}
        \caption{Gradient Comparison.}
        \label{fig_apx:grad_com1}
    \end{subfigure}
    \begin{subfigure}{0.3\linewidth}
        \centering
        \includegraphics[width=\linewidth]{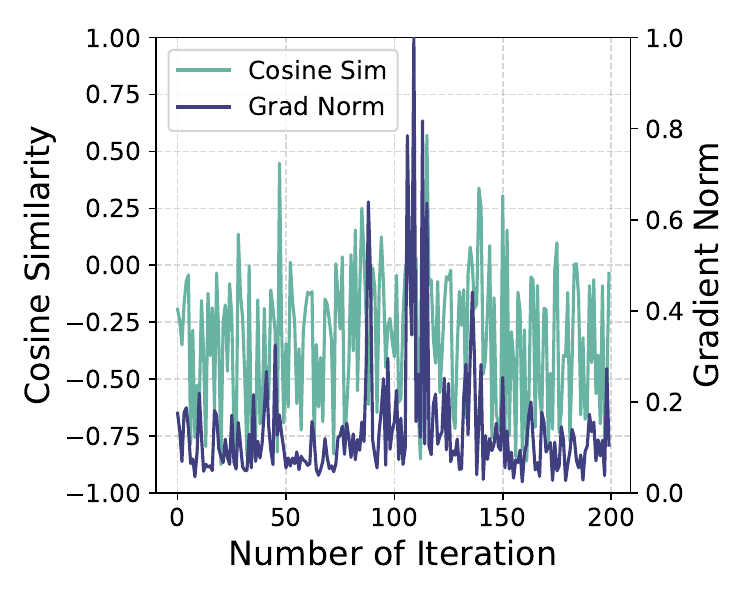}
        \caption{Original failed run.}
        \label{fig_apx:cosine_conflict2}
    \end{subfigure}
    \hspace{10pt}
    \begin{subfigure}{0.3\linewidth}
        \centering
        \includegraphics[width=\linewidth]{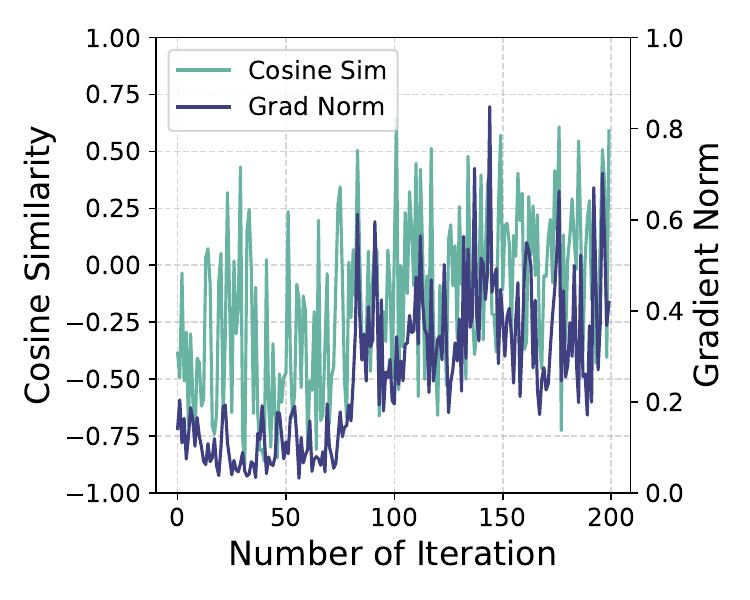}
        \caption{Stable run corrected by \mName.}
        \label{fig_apx:dsgd_grad2}
    \end{subfigure}
    \hspace{10pt}
    \begin{subfigure}{0.3\linewidth}
        \centering
        \includegraphics[width=\linewidth]{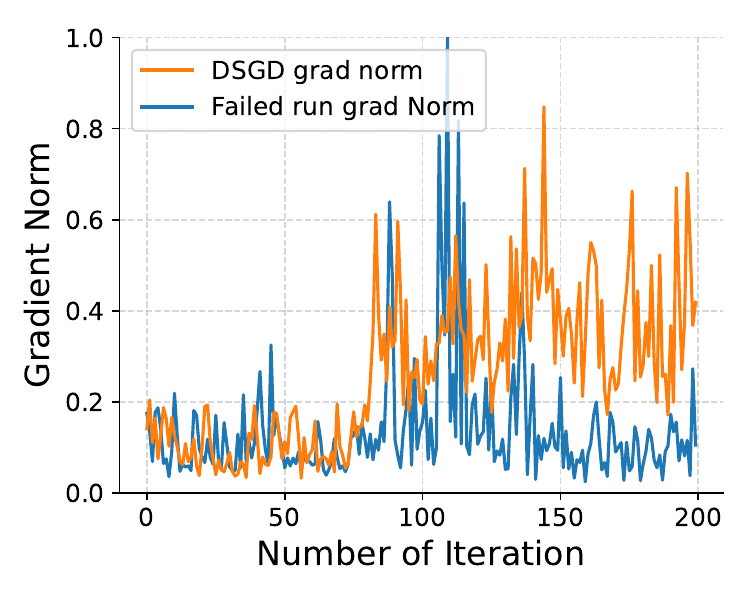}
        \caption{Gradient Comparison.}
        \label{fig_apx:grad_com2}
    \end{subfigure}
    \caption{Gradient norm comparison for RoBERTa-large on COPA (seed 132 and seed 456). (a), (d) Original failed runs. (b), (e) Stable runs corrected by \mName. In the failed run, cosine similarity reveals severe gradient conflicts. Negative values occur in 90\% of iterations, with 30\% falling below -0.5. Applying \mName does not significantly change these directions, as it only scales the magnitudes of correctly classified examples. However, gradient norms increase significantly indicating that \mName effectively mitigates gradient cancellation. Figures (c) and (e) show the gradient norm comparisons.}
    \label{fig_apx:failed_suc_comparison}
\end{figure*}

\begin{figure*}[h]
    \centering
    \begin{subfigure}{0.3\linewidth}
        \centering
        \includegraphics[width=\linewidth]{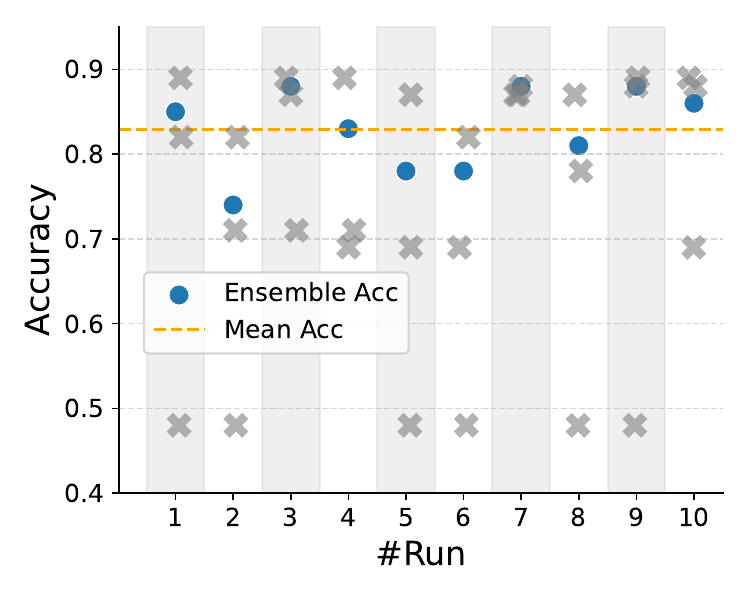}
        \caption{ENS ($\times 3$) components.}
        \label{fig:ens3_appendix}
    \end{subfigure}
    \begin{subfigure}{0.3\linewidth}
        \centering
        \includegraphics[width=\linewidth]{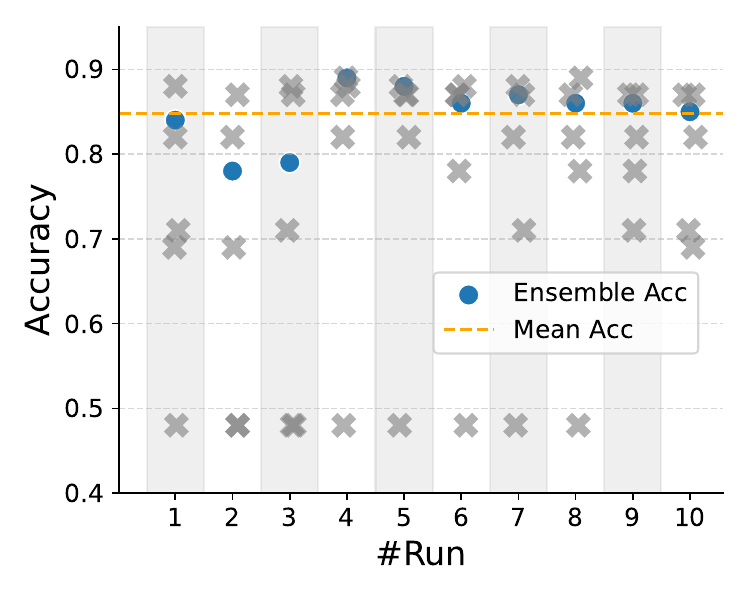}
        \caption{ENS ($\times 5$) components.}
        \label{fig:ens5_appendix}
    \end{subfigure}
    \caption{Bagging ensemble collapsed and degenerated solution proportions in individual run on COPA dataset.}
    \label{fig:ens_appendix}
\end{figure*}

% \begin{figure*}[h]
%     \centering
%     \begin{subfigure}{0.4\linewidth}
%         \centering
%         \includegraphics[width=\linewidth]{figures/PEFT_only.pdf}
%         \caption{PEFT methods accuracy.}
%         \label{fig:peft_only}
%     \end{subfigure}
%     \begin{subfigure}{0.4\linewidth}
%         \centering
%         \includegraphics[width=\linewidth]{figures/FFT_DSGD_only.pdf}
%         \caption{FFT and \mName accuracy.}
%         \label{fig:fft_dsgd}
%     \end{subfigure}
%     \caption{Bagging ensemble collapsed and degenerated solution proportions in individual run on COPA dataset.}
% \end{figure*}

\begin{figure*}[h]
    \centering
    \begin{subfigure}{0.32\linewidth}
        \centering
        \includegraphics[width=\linewidth]{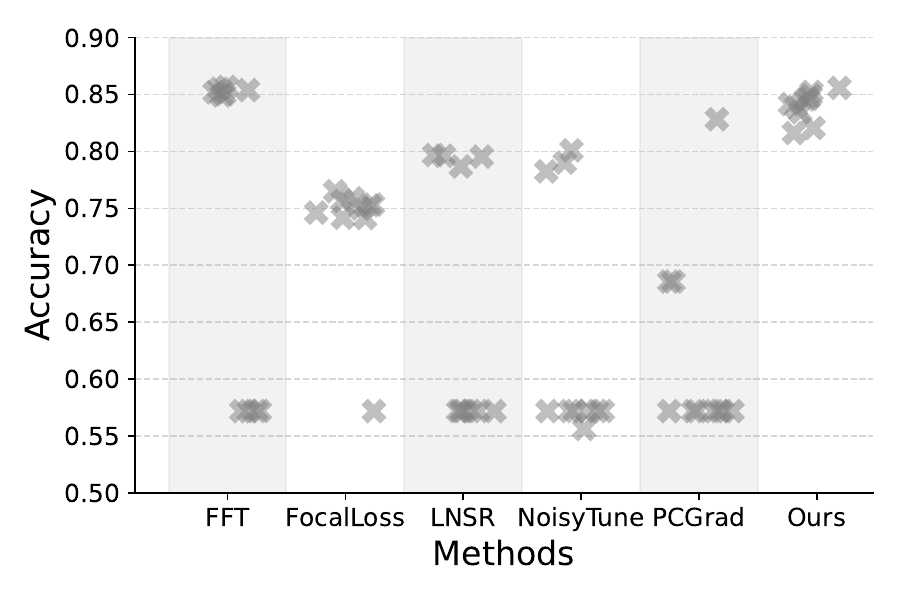}
        \caption{MultiRC.}
    \end{subfigure}
    \begin{subfigure}{0.32\linewidth}
        \centering
        \includegraphics[width=\linewidth]{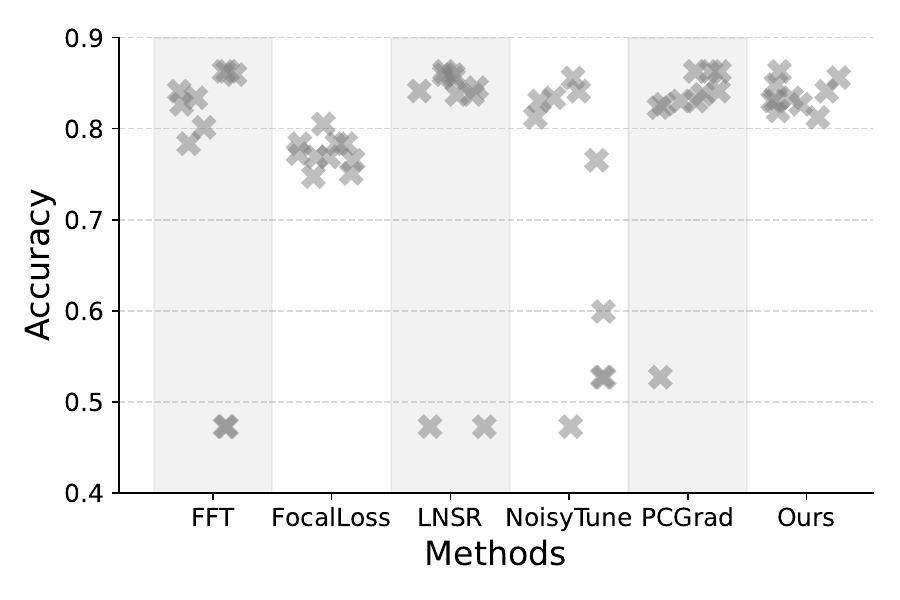}
        \caption{RTE.}
    \end{subfigure}
    \begin{subfigure}{0.32\linewidth}
        \centering
        \includegraphics[width=\linewidth]{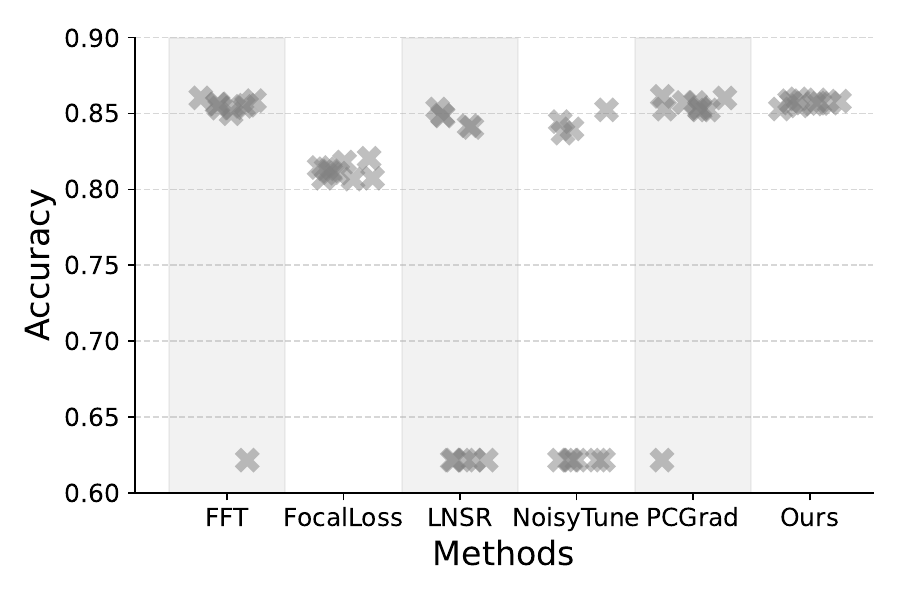}
        \caption{BoolQ.}
    \end{subfigure}
    \begin{subfigure}{0.32\linewidth}
        \centering
        \includegraphics[width=\linewidth]{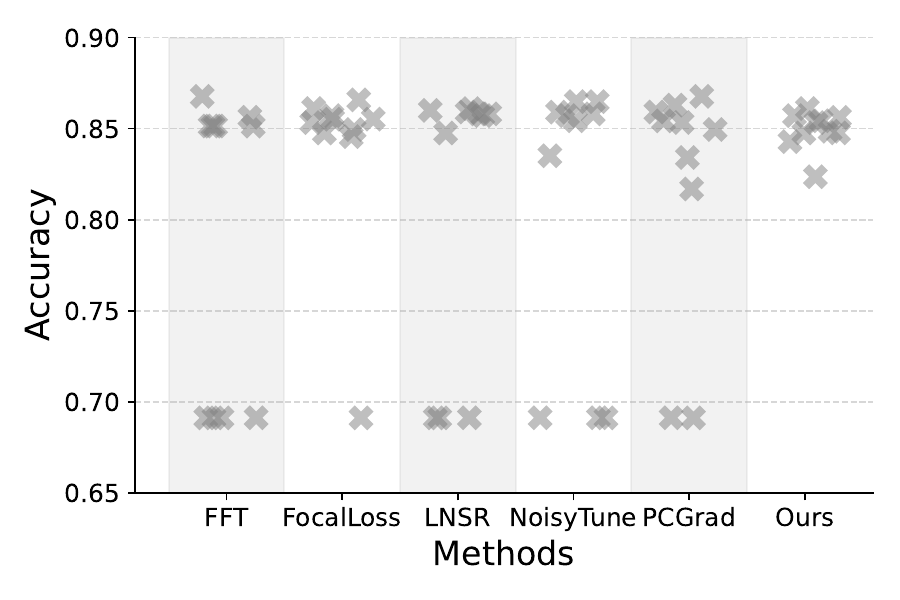}
        \caption{COLA.}
    \end{subfigure}
    \begin{subfigure}{0.32\linewidth}
        \centering
        \includegraphics[width=\linewidth]{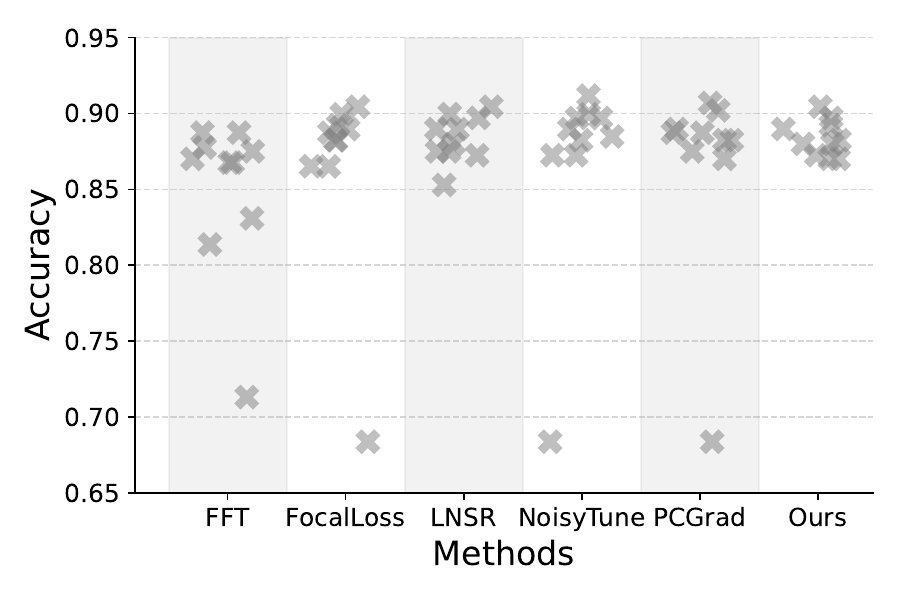}
        \caption{MRPC.}
    \end{subfigure}
    \begin{subfigure}{0.32\linewidth}
        \centering
        \includegraphics[width=\linewidth]{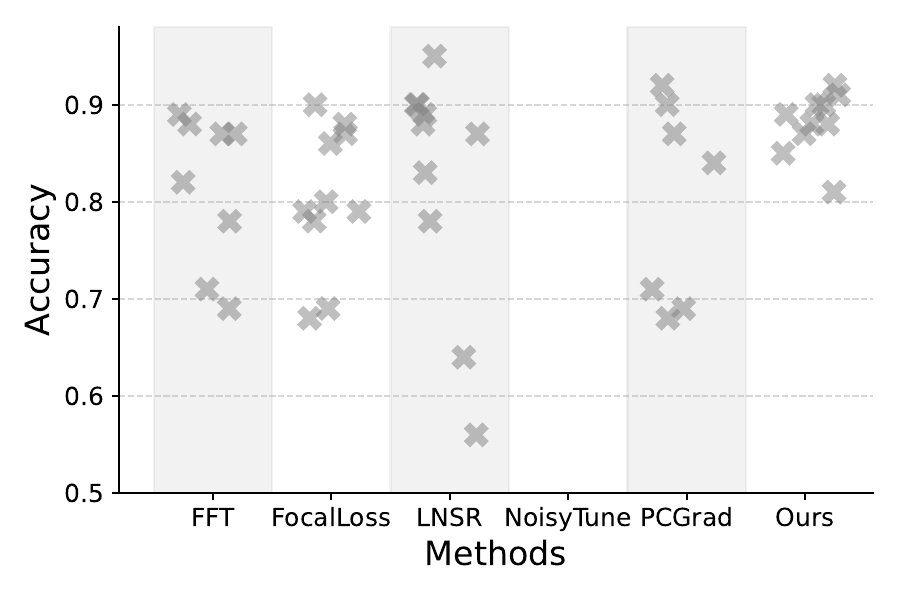}
        \caption{COPA.}
    \end{subfigure}
    \caption{Accuracies of 10 random runs by \texttt{Roberta-Large}.}
    \label{fig:roberta_appex_10_runs}
\end{figure*}

\begin{figure*}[h]
    \centering
    \begin{subfigure}{0.32\linewidth}
        \centering
        \includegraphics[width=\linewidth]{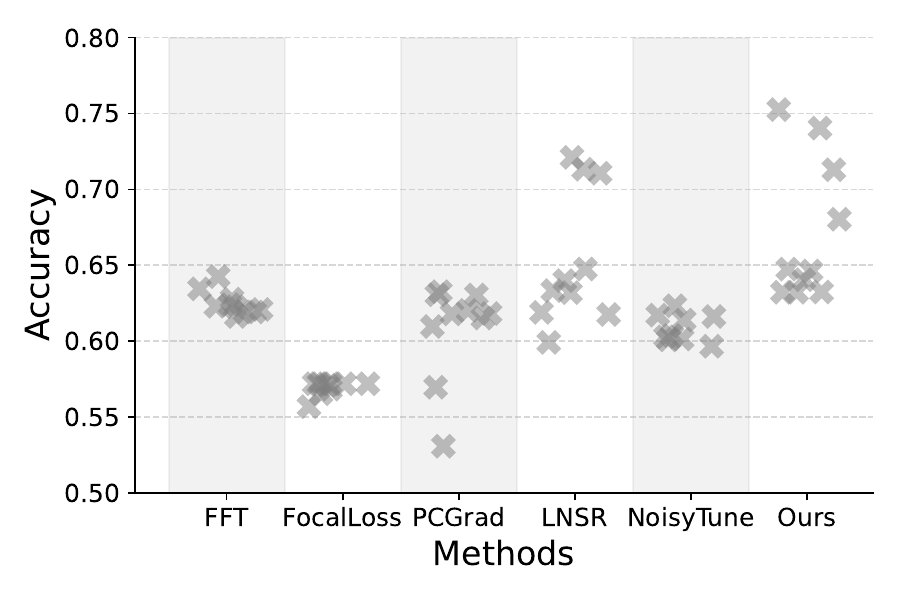}
        \caption{MultiRC.}
    \end{subfigure}
    \begin{subfigure}{0.32\linewidth}
        \centering
        \includegraphics[width=\linewidth]{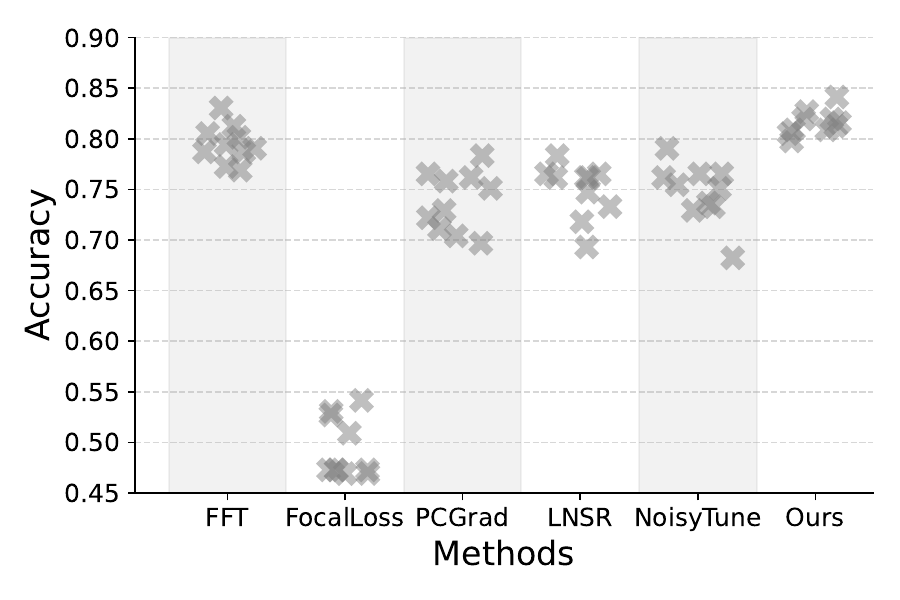}
        \caption{RTE.}
    \end{subfigure}
    \begin{subfigure}{0.32\linewidth}
        \centering
        \includegraphics[width=\linewidth]{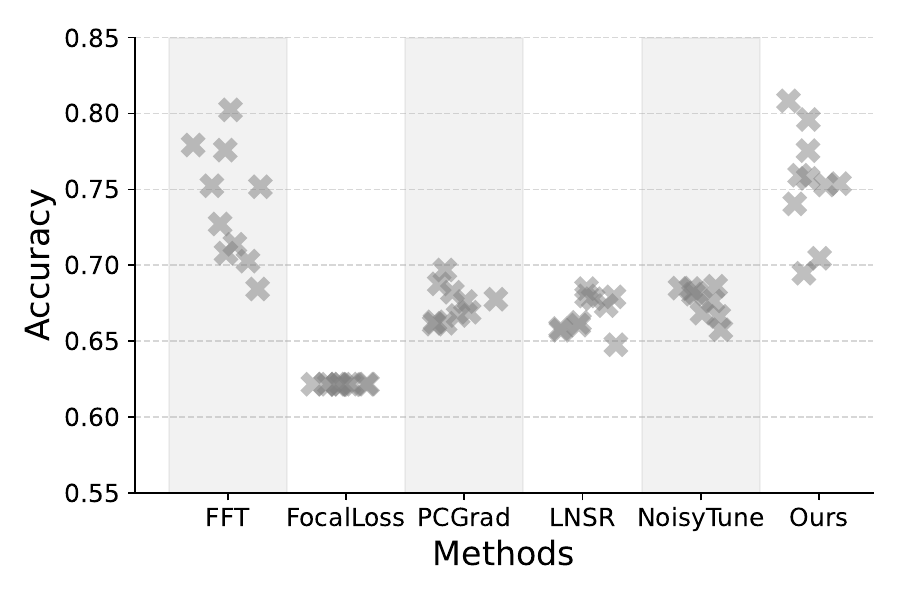}
        \caption{BoolQ.}
    \end{subfigure}
    \begin{subfigure}{0.32\linewidth}
        \centering
        \includegraphics[width=\linewidth]{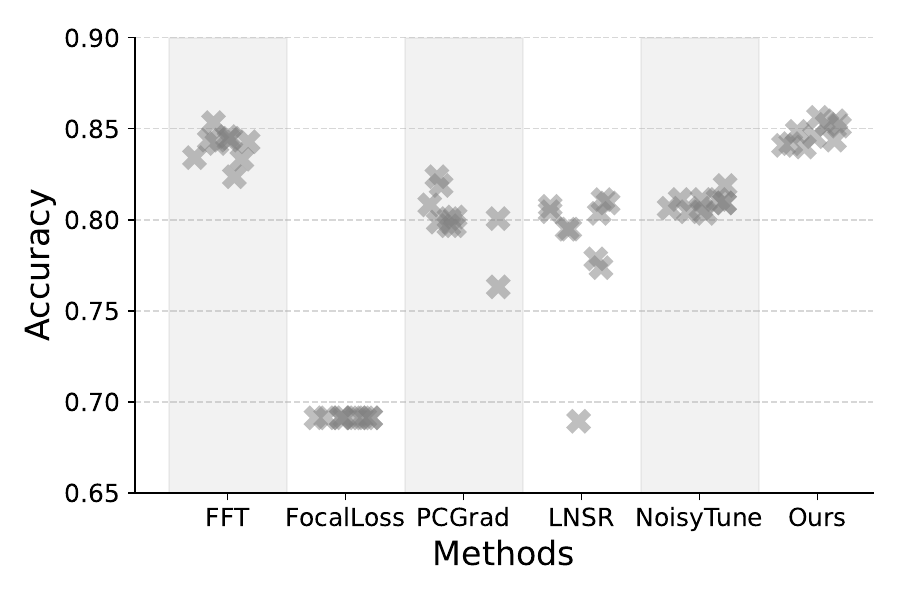}
        \caption{COLA.}
    \end{subfigure}
    \begin{subfigure}{0.32\linewidth}
        \centering
        \includegraphics[width=\linewidth]{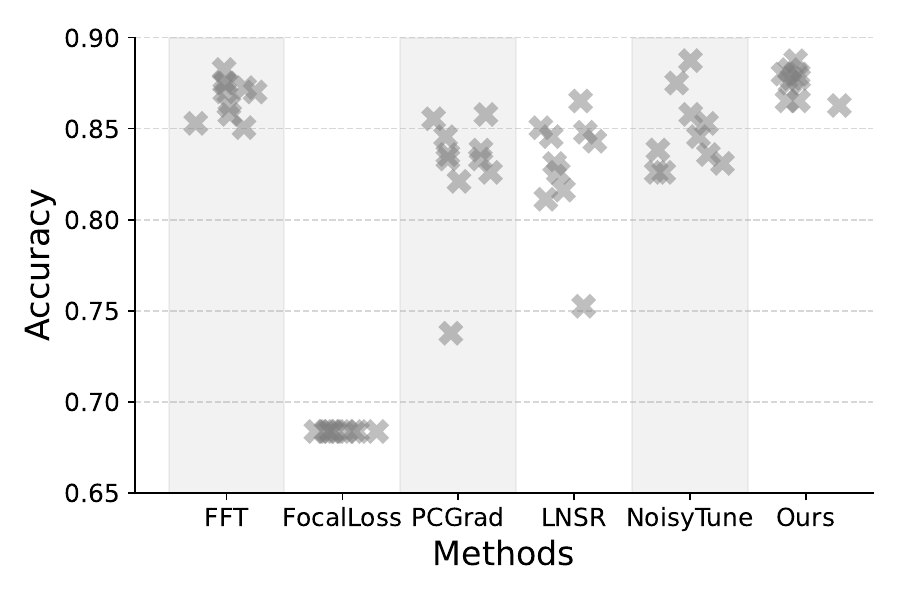}
        \caption{MRPC.}
    \end{subfigure}
    \begin{subfigure}{0.32\linewidth}
        \centering
        \includegraphics[width=\linewidth]{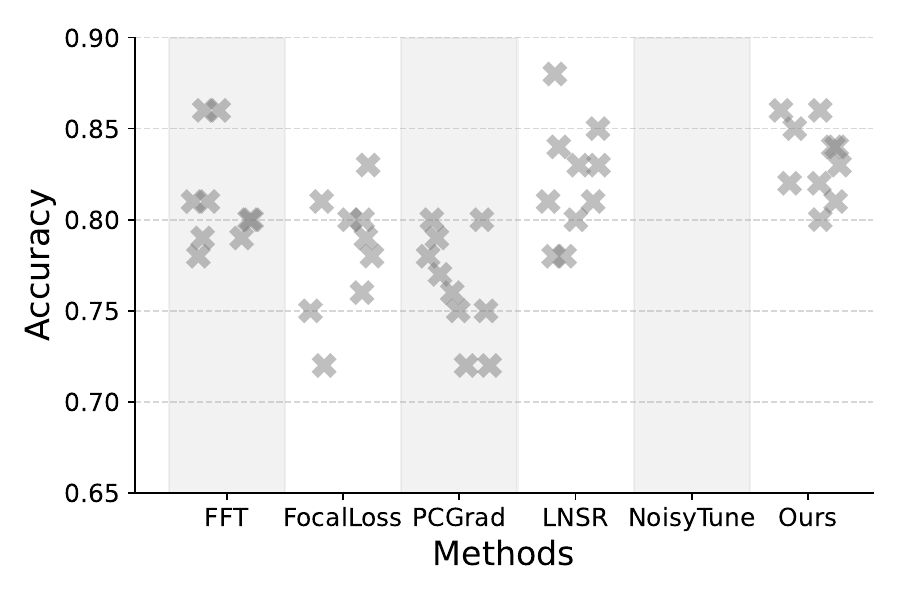}
        \caption{COPA.}
    \end{subfigure}
    \caption{Accuracies of 10 random runs by \texttt{Llama-3.2-1B}.}
    \label{fig:llama_appex_10_runs}
\end{figure*}

\begin{figure*}[h]
    \centering
    \begin{subfigure}{0.32\linewidth}
        \centering
        \includegraphics[width=\linewidth]{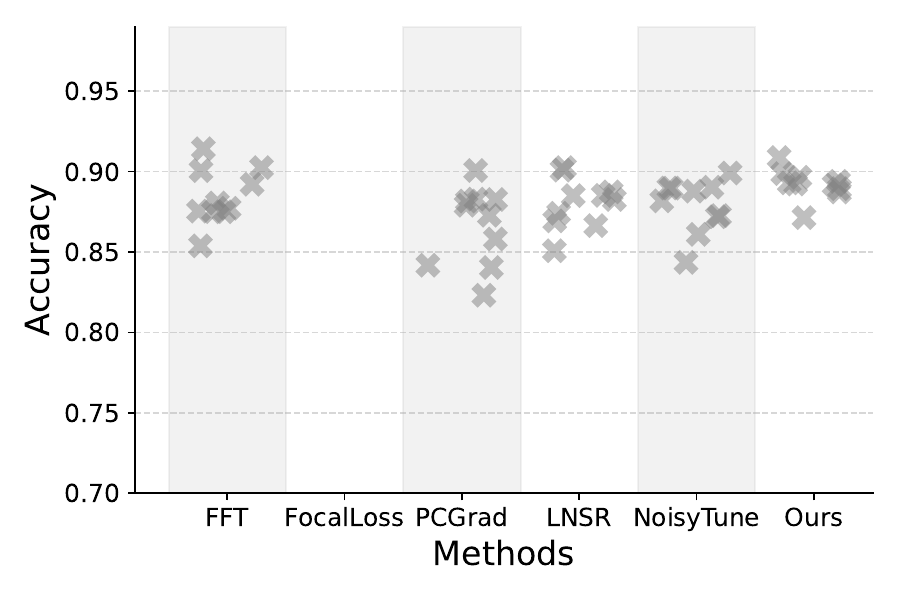}
        \caption{CIFAR-10 Long-tailed 0.01.}
    \end{subfigure}
    \begin{subfigure}{0.32\linewidth}
        \centering
        \includegraphics[width=\linewidth]{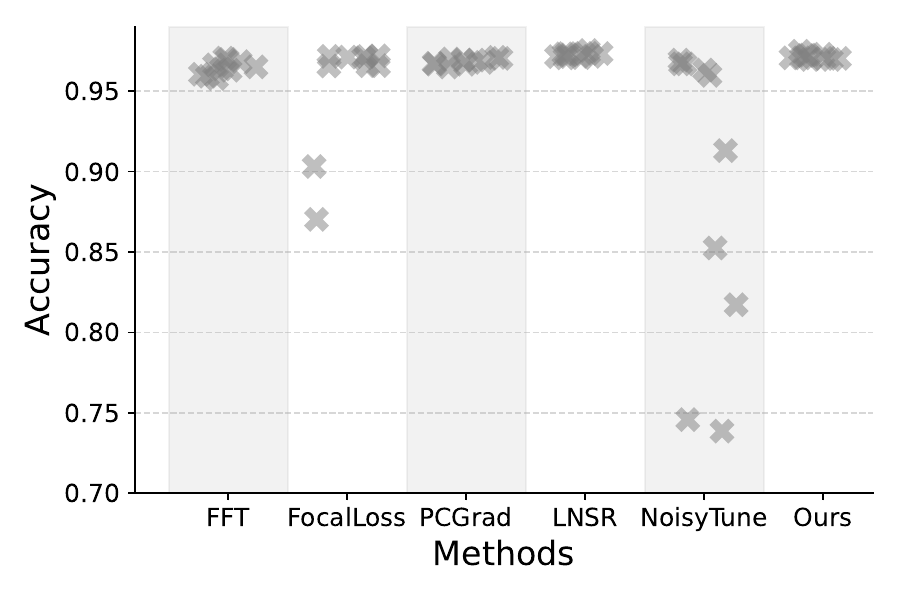}
        \caption{CIFAR-10 Long-tailed 0.5.}
    \end{subfigure}
    \begin{subfigure}{0.32\linewidth}
        \centering
        \includegraphics[width=\linewidth]{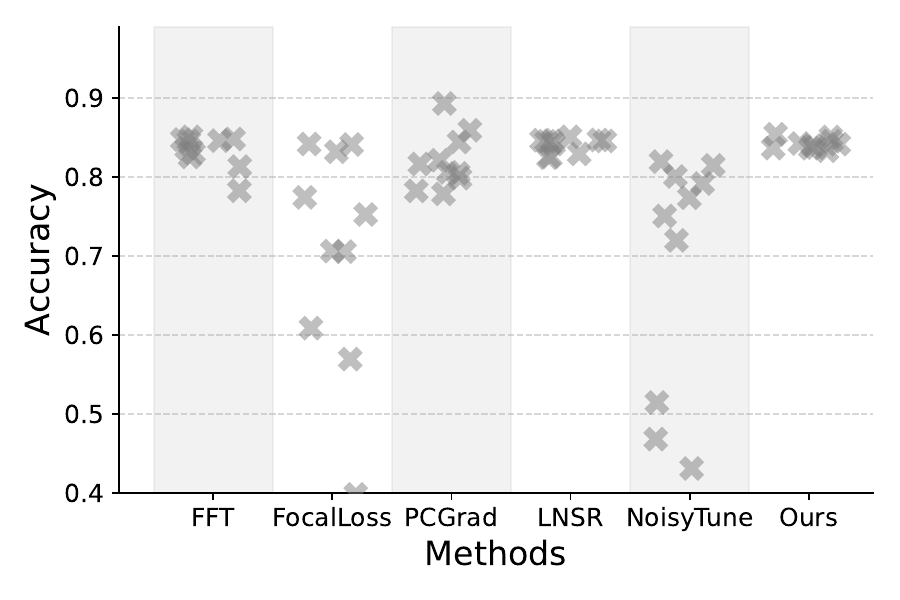}
        \caption{CIFAR-10 Step 0.01.}
    \end{subfigure}
    \begin{subfigure}{0.32\linewidth}
        \centering
        \includegraphics[width=\linewidth]{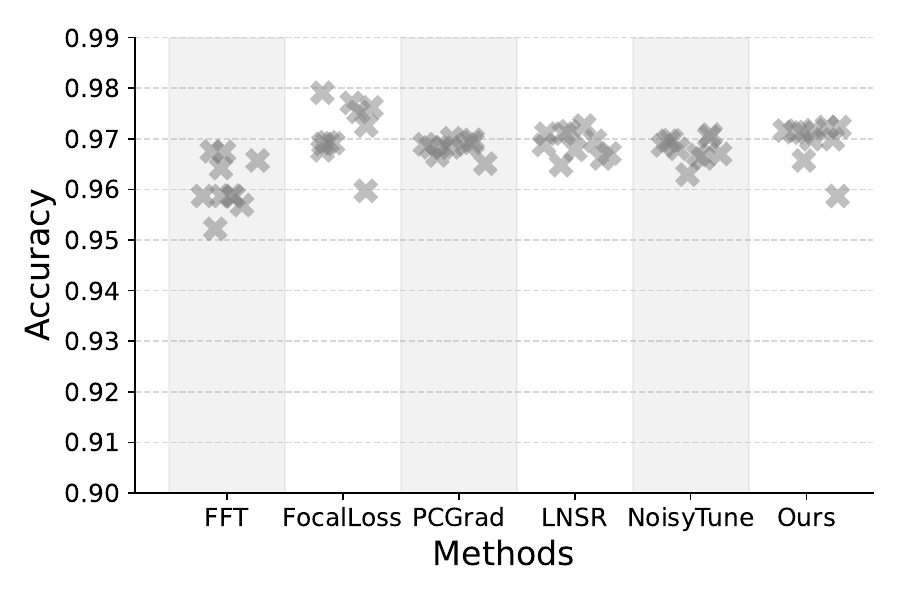}
        \caption{CIFAR-10 Step 0.5.}
    \end{subfigure}
    \begin{subfigure}{0.32\linewidth}
        \centering
        \includegraphics[width=\linewidth]{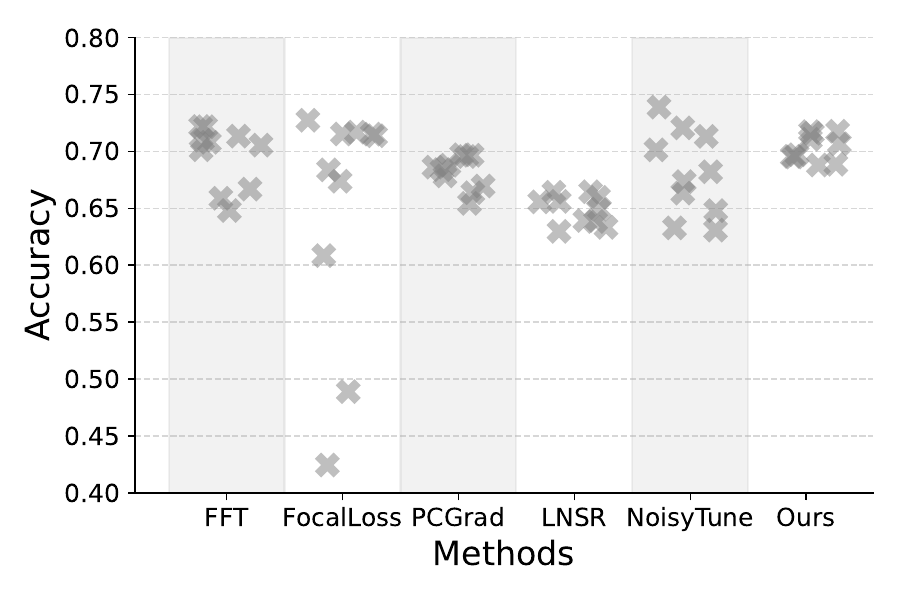}
        \caption{CIFAR-100 Long-tailed 0.01.}
    \end{subfigure}
    \begin{subfigure}{0.32\linewidth}
        \centering
        \includegraphics[width=\linewidth]{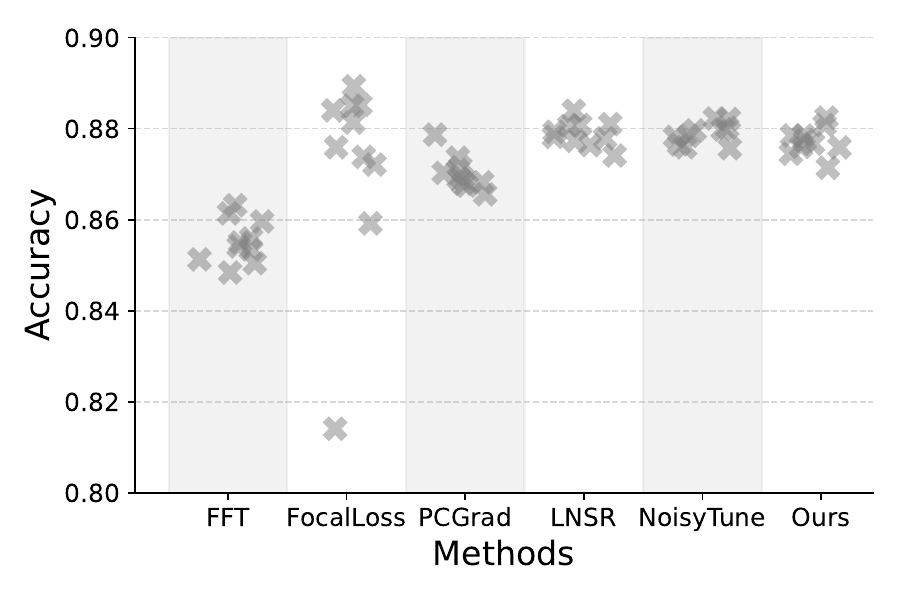}
        \caption{CIFAR-100 Long-tailed 0.5.}
    \end{subfigure}
    \begin{subfigure}{0.32\linewidth}
        \centering
        \includegraphics[width=\linewidth]{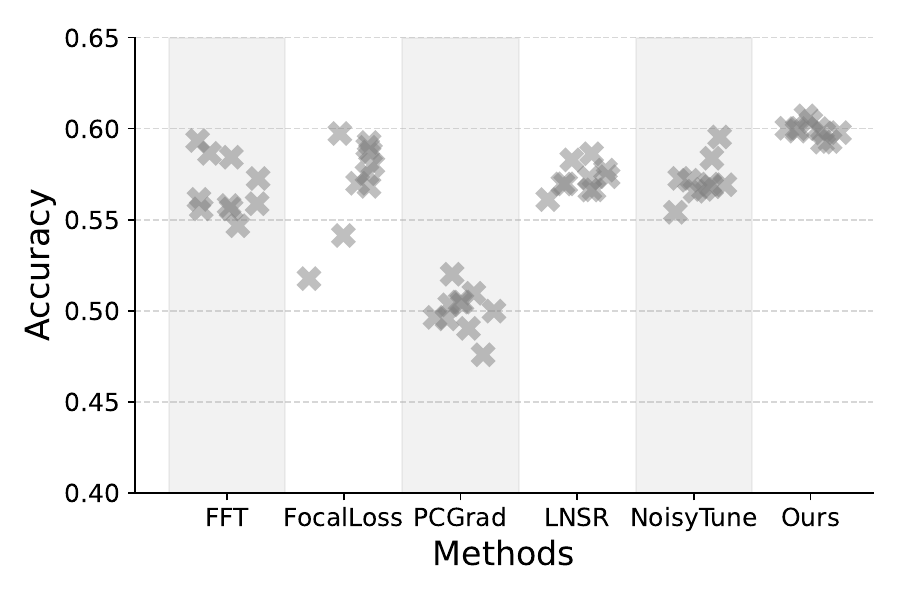}
        \caption{CIFAR-100 Step 0.01.}
    \end{subfigure}
    \begin{subfigure}{0.32\linewidth}
        \centering
        \includegraphics[width=\linewidth]{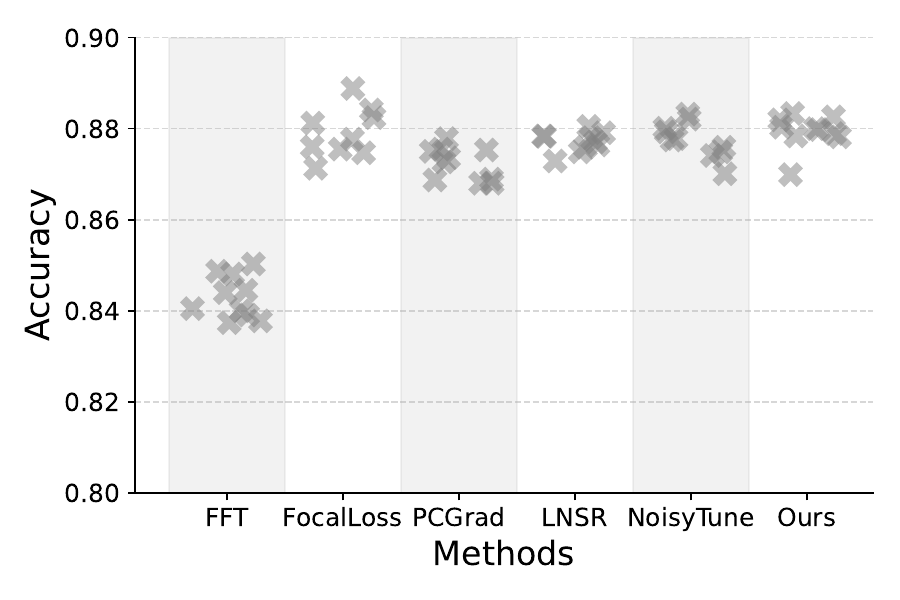}
        \caption{CIFAR-100 Step 0.5.}
    \end{subfigure}
    \caption{Accuracies of 10 random runs by \texttt{ViT-base}.}
    \label{fig:vit_appex_10_runs}
\end{figure*}

%% file: sections/relatedwork.tex
\section{Related Works} \label{sections/related_work}
Fine-tuning instability remains an open challenge, with prior work focusing on diagnosis rather than efficient solutions. 
\cite{devlin2019bert,dodge2020fine,lee2019mixout} identified two potential reasons for the observed instability: catastrophic forgetting and small size of the fine-tuning datasets. Thus increasing training size will be helpful in improving stability. However, collecting and annotating new data is super challenging in some domains. Thus, this method is infeasible in many domain tasks and its effectiveness is not guarantied. 
\cite{mosbach2020stability} attributes failed runs to unsuccessful optimization. They stated that observed fine-tuning instability is caused by optimization difficulties that lead to vanishing gradients. They proposed guidelines for improving stability of fine-tuning BERT which are using small learning rates with bias correction optimizer (like ADAM) and increase the number of iterations considerably and train to (almost) zero training loss. We followed their guidelines in FFT in our experiment which still show significant performance variance. 
While these works provide guidelines in improving stability of fine-tuning, they are either infeasible or less effective in real-world scenarios.  

Ensemble methods offer the most effective mitigation, providing the greatest improvements in performance and reductions in variance compared to experimental setup enhancements or data augmentation, as supported by numerous theoretical analyses~\cite{wang2020wisdom, wang-etal-2023-two}. \cite{wang2020wisdom} studies model instability by defining consistency of a learning model in the context of periodic retraining of deployed models where the outputs from successive generations of the models might not agree on the correct labels assigned to the same input. It proposes dynamic snapshot ensemble to improve the consistency with theoretical supports. However, this method requires sequencial training of those components which is highly inefficient in time-intensive scenarios. 
Later \cite{wang-etal-2023-two} decomposes high performance variance of LPMs in fine-tuning into variance due to sampling and variance due to optimization. They perform theoretical and empirical analysis using bagging ensemble to decrease variance due to optimization. However, ensemble suffers from high computational cost as discussed in Section~\ref{subsec:computational_cost}.

Weighted averaging
\cite{izmailov2018averaging, gao2022revisiting, madhyastha2019model, nishida-etal-2025-instability}
is a technique that averaging nearby checkpoints to produce the final model which is empirically and theoretically shown efficiency in reducing model variations in instability tasks by avoiding saddle points in optimization landscape~\cite{madhyastha2019model}. 
\cite{nishida-etal-2025-instability} attributes model stability varies across different runs due to the non-determinism factors such as weight initialization, dropout, stochastic data shuffling. They investigate two post-hoc checkpoint integration methods:
checkpoint averaging and ensemble, to reduce performance volatility. However, these methods require huge storage space to save those checkpoints which prohibits LPMs from deploying in real-time systems.

Noise-based techniques~\cite{hua2021noise, wu2022noisytune} typically seek to improve generalization and stability of LPMs in fine-tuning by injecting noise but lack a targeted mechanism to directly stabilize the fine-tuning optimization process.
The noise injection methods add randomness to everywhere which  does not stabilize the sensitive directions in the parameter space that cause seed divergence but can actually increase divergence by starting from more varied points. The added noise may destroy some of the transferrable knowledge leading to reduced accuracy.
Instead, our method attributes the cause of fine-tuning instability as the gradient conflicts between training examples and proposes a solution that directly targets the optimality during fine-tuning. It ensures escape from the collapsed state regardless of starting points to obtain better convergence and a tighter stability bound, leading to consistently better accuracy and variance. 

Despite no method has been proposed to solve gradient cancellation in fine-tuning, gradient conflicts resolvers are proposed with the aim of resolving the cancellation in multitask learning where the objective functions of individual tasks conflicts to other making the learning inefficient~\cite{sener2018multi, yu2020gradient, liu2021conflict, chen2020just}. We adapted these methods to solve gradient cancellation in multilabel fine-tuning learning. However, they acts when gradients have significant negative cosine similarity (gradient cancellation), but collapsed state can happen even with orthogonal gradients or small-magnitude gradients for one task. In addition, these methods require two backward passes to compute the gradient separately and therefore nearly doubles the training time.
In contrast, our method \mName actively adjusts gradients based on learning outcomes that do not depend on the detection of actual gradient cancellation, making it a more robust and proactive defense against task collapse and seed-induced instability.

%% file: tables/joint_task.tex
\begin{table*}[t]
\centering
\begin{adjustbox}{max width=\linewidth}
\begin{tabular}{l|cccccc|cc}
\toprule
          \texttt{Roberta-large}& \multicolumn{2}{c}{MRPC + RTE} & \multicolumn{2}{c}{MRPC + CoLA} & \multicolumn{2}{c|}{RTE + CoLA} &  ACC($\uparrow$) &STD($\downarrow$)\\
          \cmidrule(r){2-3} \cmidrule(r){4-5} \cmidrule(r){6-7}
          & MRPC                  & RTE  & MRPC          & CoLA          & RTE          & CoLA           & &\\
\midrule
FFT       & \multicolumn{1}{c}{85.22 $\pm$ 8.91 }  &      76.03 $\pm$ 12.41&               77.33 $\pm$ 10.64&               77.30 $\pm$ 8.63&              68.30 $\pm$ 16.05&               80.51 $\pm$ 7.88 & 77.44&10.75\\
PCGrad    &                       \textbf{88.63 $\pm$ 0.67}&      \textbf{83.25 $\pm$ 1.99}&               76.67 $\pm$  10.70&               75.66 $\pm$ 8.44&              \underline{80.94 $\pm$ 10.10}&               83.99 $\pm$ 5.26& \underline{81.52}&\underline{6.19}\\
LNSR      &                       81.18 $\pm$ 11.03&      69.75 $\pm$ 14.69&               \underline{88.01 $\pm$ 6.93}&               \underline{83.74 $\pm$ 5.17}&              79.24 $\pm$ 11.54&               83.59 $\pm$ 5.12 & 80.91&9.08\\
NoisyTune &                       84.78 $\pm$ 8.87&      75.02 $\pm$ 12.68&               70.80 $\pm$ 6.76&               71.74 $\pm$ 5.46&              68.16 $\pm$ 14.94&               81.14 $\pm$ 6.86 & 75.27&9.26\\
SWA       &                       79.68 $\pm$ 9.97&      65.56 $\pm$ 14.62&               79.63 $\pm$ 10.86&               79.01 $\pm$ 8.54&              76.79 $\pm$ 12.55&               \underline{85.54 $\pm$ 1.04}& 77.68&9.59\\
\rowcolor{gray!20} \mName         &      \underline{88.14 $\pm$ 1.20}&      \underline{81.73 $\pm$ 1.56}&               \textbf{88.95 $\pm$ 0.80}&               \textbf{85.21 $\pm$ 0.57}&              \textbf{83.10 $\pm$ 2.56}&               \textbf{85.35 $\pm$ 0.47} & \textbf{85.41}&\textbf{1.19}\\
\bottomrule
\end{tabular}
\end{adjustbox}
\caption{Comparison with single-learner baselines on multi-tasks learning settings. 
% \textcolor{red}{add FocalLoss.}
}
\label{table:multitask_result}
\end{table*}

%% file: ref.bib
@article{lee2019mixout,
  title={Mixout: Effective regularization to finetune large-scale pretrained language models},
  author={Lee, Cheolhyoung and Cho, Kyunghyun and Kang, Wanmo},
  journal={arXiv preprint arXiv:1909.11299},
  year={2019}
}

@article{gupta2023coverage,
  title={Coverage-based example selection for in-context learning},
  author={Gupta, Shivanshu and Gardner, Matt and Singh, Sameer},
  journal={arXiv preprint arXiv:2305.14907},
  year={2023}
}

@article{bethard2022we,
  title={We need to talk about random seeds},
  author={Bethard, Steven},
  journal={arXiv preprint arXiv:2210.13393},
  year={2022}
}

@article{han2024parameter,
  title={Parameter-efficient fine-tuning for large models: A comprehensive survey},
  author={Han, Zeyu and Gao, Chao and Liu, Jinyang and Zhang, Jeff and Zhang, Sai Qian},
  journal={arXiv preprint arXiv:2403.14608},
  year={2024}
}

@article{wang2019superglue,
  title={Superglue: A stickier benchmark for general-purpose language understanding systems},
  author={Wang, Alex and Pruksachatkun, Yada and Nangia, Nikita and Singh, Amanpreet and Michael, Julian and Hill, Felix and Levy, Omer and Bowman, Samuel},
  journal={Advances in neural information processing systems},
  volume={32},
  year={2019}
}

@article{touvron2023llama,
  title={Llama: Open and efficient foundation language models},
  author={Touvron, Hugo and Lavril, Thibaut and Izacard, Gautier and Martinet, Xavier and Lachaux, Marie-Anne and Lacroix, Timoth{\'e}e and Rozi{\`e}re, Baptiste and Goyal, Naman and Hambro, Eric and Azhar, Faisal and others},
  journal={arXiv preprint arXiv:2302.13971},
  year={2023}
}

@article{liu2019roberta,
  title={Roberta: A robustly optimized bert pretraining approach},
  author={Liu, Yinhan and Ott, Myle and Goyal, Naman and Du, Jingfei and Joshi, Mandar and Chen, Danqi and Levy, Omer and Lewis, Mike and Zettlemoyer, Luke and Stoyanov, Veselin},
  journal={arXiv preprint arXiv:1907.11692},
  year={2019}
}

@article{cao2019learning,
  title={Learning imbalanced datasets with label-distribution-aware margin loss},
  author={Cao, Kaidi and Wei, Colin and Gaidon, Adrien and Arechiga, Nikos and Ma, Tengyu},
  journal={Advances in neural information processing systems},
  volume={32},
  year={2019}
}

@article{dosovitskiy2020image,
  title={An image is worth 16x16 words: Transformers for image recognition at scale},
  author={Dosovitskiy, Alexey},
  journal={arXiv preprint arXiv:2010.11929},
  year={2020}
}

@article{yu2020gradient,
  title={Gradient surgery for multi-task learning},
  author={Yu, Tianhe and Kumar, Saurabh and Gupta, Abhishek and Levine, Sergey and Hausman, Karol and Finn, Chelsea},
  journal={Advances in neural information processing systems},
  volume={33},
  pages={5824--5836},
  year={2020}
}

@article{liu2021conflict,
  title={Conflict-averse gradient descent for multi-task learning},
  author={Liu, Bo and Liu, Xingchao and Jin, Xiaojie and Stone, Peter and Liu, Qiang},
  journal={Advances in Neural Information Processing Systems},
  volume={34},
  pages={18878--18890},
  year={2021}
}

@inproceedings{wang-etal-2023-two,
    title = "Two-Stage Fine-Tuning for Improved Bias and Variance for Large Pretrained Language Models",
    author = "Wang, Lijing  and
      Li, Yingya  and
      Miller, Timothy  and
      Bethard, Steven  and
      Savova, Guergana",
    editor = "Rogers, Anna  and
      Boyd-Graber, Jordan  and
      Okazaki, Naoaki",
    booktitle = "Proceedings of the 61st Annual Meeting of the Association for Computational Linguistics (Volume 1: Long Papers)",
    month = jul,
    year = "2023",
    address = "Toronto, Canada",
    publisher = "Association for Computational Linguistics",
    url = "https://aclanthology.org/2023.acl-long.877/",
    doi = "10.18653/v1/2023.acl-long.877",
    pages = "15746--15761"
}

@article{wang2020wisdom,
  title={Wisdom of the ensemble: Improving consistency of deep learning models},
  author={Wang, Lijing and Ghosh, Dipanjan and Gonzalez Diaz, Maria and Farahat, Ahmed and Alam, Mahbubul and Gupta, Chetan and Chen, Jiangzhuo and Marathe, Madhav},
  journal={Advances in Neural Information Processing Systems},
  volume={33},
  pages={19750--19761},
  year={2020}
}

@inproceedings{hardt2016train,
  title={Train faster, generalize better: Stability of stochastic gradient descent},
  author={Hardt, Moritz and Recht, Ben and Singer, Yoram},
  booktitle={International conference on machine learning},
  pages={1225--1234},
  year={2016},
  organization={PMLR}
}

@article{hu2022lora,
  title={Lora: Low-rank adaptation of large language models.},
  author={Hu, Edward J and Shen, Yelong and Wallis, Phillip and Allen-Zhu, Zeyuan and Li, Yuanzhi and Wang, Shean and Wang, Lu and Chen, Weizhu and others},
  journal={ICLR},
  volume={1},
  number={2},
  pages={3},
  year={2022}
}

@book{10.5555/2670022,
author = {Nesterov, Yurii},
title = {Introductory Lectures on Convex Optimization: A Basic Course},
year = {2014},
isbn = {1461346916},
publisher = {Springer Publishing Company, Incorporated},
edition = {1}
}

@article{mosbach2020stability,
  title={On the Stability of Fine-tuning BERT: Misconceptions},
  author={Mosbach, Marius and Andriushchenko, Maksym and Klakow, Dietrich},
  journal={Explanations, and Strong Baselines. arXiv},
  year={2020}
}

@misc{bui2025assessing,
      title={Assessing the Macro and Micro Effects of Random Seeds on Fine-Tuning Large Language Models}, 
      author={Nghia Bui and Guergana Savova and Lijing Wang},
      year={2025},
      eprint={2503.07329},
      archivePrefix={arXiv},
      primaryClass={cs.CL},
      url={https://arxiv.org/abs/2503.07329}, 
}

@inproceedings{mosbach-2023-analyzing,
    title = "Analyzing Pre-trained and Fine-tuned Language Models",
    author = "Mosbach, Marius",
    editor = "Elazar, Yanai  and
      Ettinger, Allyson  and
      Kassner, Nora  and
      Ruder, Sebastian  and
      A. Smith, Noah",
    booktitle = "Proceedings of the Big Picture Workshop",
    month = dec,
    year = "2023",
    address = "Singapore",
    publisher = "Association for Computational Linguistics",
    url = "https://aclanthology.org/2023.bigpicture-1.10/",
    doi = "10.18653/v1/2023.bigpicture-1.10",
    pages = "123--134"
}

@inproceedings{wang2018glue,
  title={GLUE: A multi-task benchmark and analysis platform for natural language understanding},
  author={Wang, Alex and Singh, Amanpreet and Michael, Julian and Hill, Felix and Levy, Omer and Bowman, Samuel},
  booktitle={Proceedings of the 2018 EMNLP workshop BlackboxNLP: Analyzing and interpreting neural networks for NLP},
  pages={353--355},
  year={2018}
}

@article{loshchilov2017decoupled,
  title={Decoupled weight decay regularization},
  author={Loshchilov, Ilya and Hutter, Frank},
  journal={arXiv preprint arXiv:1711.05101},
  year={2017}
}

@inproceedings{hua2021noise,
  title={Noise stability regularization for improving BERT fine-tuning},
  author={Hua, Hang and Li, Xingjian and Dou, Dejing and Xu, Chengzhong and Luo, Jiebo},
  booktitle={Proceedings of the 2021 Conference of the North American Chapter of the Association for Computational Linguistics: Human Language Technologies},
  pages={3229--3241},
  year={2021}
}

@inproceedings{wu2022noisytune,
  title={Noisytune: A little noise can help you finetune pretrained language models better},
  author={Wu, Chuhan and Wu, Fangzhao and Qi, Tao and Huang, Yongfeng},
  booktitle={Proceedings of the 60th Annual Meeting of the Association for Computational Linguistics (Volume 2: Short Papers)},
  pages={680--685},
  year={2022}
}

@article{chen2020just,
  title={Just pick a sign: Optimizing deep multitask models with gradient sign dropout},
  author={Chen, Zhao and Ngiam, Jiquan and Huang, Yanping and Luong, Thang and Kretzschmar, Henrik and Chai, Yuning and Anguelov, Dragomir},
  journal={Advances in Neural Information Processing Systems},
  volume={33},
  pages={2039--2050},
  year={2020}
}

@inproceedings{devlin2019bert,
  title={Bert: Pre-training of deep bidirectional transformers for language understanding},
  author={Devlin, Jacob and Chang, Ming-Wei and Lee, Kenton and Toutanova, Kristina},
  booktitle={Proceedings of the 2019 conference of the North American chapter of the association for computational linguistics: human language technologies, volume 1 (long and short papers)},
  pages={4171--4186},
  year={2019}
}

@article{phang2018sentence,
  title={Sentence encoders on stilts: Supplementary training on intermediate labeled-data tasks},
  author={Phang, Jason and F{\'e}vry, Thibault and Bowman, Samuel R},
  journal={arXiv preprint arXiv:1811.01088},
  year={2018}
}

@article{zhang2023adalora,
  title={Adalora: Adaptive budget allocation for parameter-efficient fine-tuning},
  author={Zhang, Qingru and Chen, Minshuo and Bukharin, Alexander and Karampatziakis, Nikos and He, Pengcheng and Cheng, Yu and Chen, Weizhu and Zhao, Tuo},
  journal={arXiv preprint arXiv:2303.10512},
  year={2023}
}

@article{liu2022few,
  title={Few-shot parameter-efficient fine-tuning is better and cheaper than in-context learning},
  author={Liu, Haokun and Tam, Derek and Muqeeth, Mohammed and Mohta, Jay and Huang, Tenghao and Bansal, Mohit and Raffel, Colin A},
  journal={Advances in Neural Information Processing Systems},
  volume={35},
  pages={1950--1965},
  year={2022}
}

@inproceedings{cui2019class,
  title={Class-balanced loss based on effective number of samples},
  author={Cui, Yin and Jia, Menglin and Lin, Tsung-Yi and Song, Yang and Belongie, Serge},
  booktitle={Proceedings of the IEEE/CVF conference on computer vision and pattern recognition},
  pages={9268--9277},
  year={2019}
}

@article{azizzadenesheli2019regularized,
  title={Regularized learning for domain adaptation under label shifts},
  author={Azizzadenesheli, Kamyar and Liu, Anqi and Yang, Fanny and Anandkumar, Animashree},
  journal={arXiv preprint arXiv:1903.09734},
  year={2019}
}

@inproceedings{nishida-etal-2025-instability,
    title = "Instability in Downstream Task Performance During {LLM} Pretraining",
    author = "Nishida, Yuto  and
      Isonuma, Masaru  and
      Oda, Yusuke",
    editor = "Christodoulopoulos, Christos  and
      Chakraborty, Tanmoy  and
      Rose, Carolyn  and
      Peng, Violet",
    booktitle = "Findings of the Association for Computational Linguistics: EMNLP 2025",
    month = nov,
    year = "2025",
    address = "Suzhou, China",
    publisher = "Association for Computational Linguistics",
    url = "https://aclanthology.org/2025.findings-emnlp.1246/",
    doi = "10.18653/v1/2025.findings-emnlp.1246",
    pages = "22883--22895",
    ISBN = "979-8-89176-335-7"
}

@article{izmailov2018averaging,
  title={Averaging weights leads to wider optima and better generalization},
  author={Izmailov, Pavel and Podoprikhin, Dmitrii and Garipov, Timur and Vetrov, Dmitry and Wilson, Andrew Gordon},
  journal={arXiv preprint arXiv:1803.05407},
  year={2018}
}

@article{gao2022revisiting,
  title={Revisiting checkpoint averaging for neural machine translation},
  author={Gao, Yingbo and Herold, Christian and Yang, Zijian and Ney, Hermann},
  journal={arXiv preprint arXiv:2210.11803},
  year={2022}
}

@inproceedings{madhyastha2019model,
  title={On model stability as a function of random seed},
  author={Madhyastha, Pranava Swaroop and Jain, Rishabh},
  booktitle={Proceedings of the 23rd Conference on Computational Natural Language Learning (CoNLL)},
  pages={929--939},
  year={2019}
}

@article{dodge2020fine,
  title={Fine-tuning pretrained language models: Weight initializations, data orders, and early stopping},
  author={Dodge, Jesse and Ilharco, Gabriel and Schwartz, Roy and Farhadi, Ali and Hajishirzi, Hannaneh and Smith, Noah},
  journal={arXiv preprint arXiv:2002.06305},
  year={2020}
}

@inproceedings{summers2021nondeterminism,
  title={Nondeterminism and instability in neural network optimization},
  author={Summers, Cecilia and Dinneen, Michael J},
  booktitle={International Conference on Machine Learning},
  pages={9913--9922},
  year={2021},
  organization={PMLR}
}

@article{picard2021torch,
  title={Torch. manual\_seed (3407) is all you need: On the influence of random seeds in deep learning architectures for computer vision},
  author={Picard, David},
  journal={arXiv preprint arXiv:2109.08203},
  year={2021}
}

@inproceedings{hidey-etal-2022-reducing,
    title = "Reducing Model Churn: Stable Re-training of Conversational Agents",
    author = "Hidey, Christopher  and
      Liu, Fei  and
      Goel, Rahul",
    editor = "Lemon, Oliver  and
      Hakkani-Tur, Dilek  and
      Li, Junyi Jessy  and
      Ashrafzadeh, Arash  and
      Garcia, Daniel Hern{\'a}ndez  and
      Alikhani, Malihe  and
      Vandyke, David  and
      Du{\v{s}}ek, Ond{\v{r}}ej",
    booktitle = "Proceedings of the 23rd Annual Meeting of the Special Interest Group on Discourse and Dialogue",
    month = sep,
    year = "2022",
    address = "Edinburgh, UK",
    publisher = "Association for Computational Linguistics",
    url = "https://aclanthology.org/2022.sigdial-1.2/",
    doi = "10.18653/v1/2022.sigdial-1.2",
    pages = "14--25"
}

@inproceedings{pecher-etal-2024-fighting,
    title = "Fighting Randomness with Randomness: Mitigating Optimisation Instability of Fine-Tuning using Delayed Ensemble and Noisy Interpolation",
    author = "Pecher, Branislav  and
      Cegin, Jan  and
      Belanec, Robert  and
      Simko, Jakub  and
      Srba, Ivan  and
      Bielikova, Maria",
    editor = "Al-Onaizan, Yaser  and
      Bansal, Mohit  and
      Chen, Yun-Nung",
    booktitle = "Findings of the Association for Computational Linguistics: EMNLP 2024",
    month = nov,
    year = "2024",
    address = "Miami, Florida, USA",
    publisher = "Association for Computational Linguistics",
    url = "https://aclanthology.org/2024.findings-emnlp.644/",
    doi = "10.18653/v1/2024.findings-emnlp.644",
    pages = "11005--11044"
}

@article{sener2018multi,
  title={Multi-task learning as multi-objective optimization},
  author={Sener, Ozan and Koltun, Vladlen},
  journal={Advances in neural information processing systems},
  volume={31},
  year={2018}
}

@article{buda2018systematic,
  title={A systematic study of the class imbalance problem in convolutional neural networks},
  author={Buda, Mateusz and Maki, Atsuto and Mazurowski, Maciej A},
  journal={Neural networks},
  volume={106},
  pages={249--259},
  year={2018},
  publisher={Elsevier}
}

@article{zhang2020revisiting,
  title={Revisiting few-sample BERT fine-tuning},
  author={Zhang, Tianyi and Wu, Felix and Katiyar, Arzoo and Weinberger, Kilian Q and Artzi, Yoav},
  journal={arXiv preprint arXiv:2006.05987},
  year={2020}
}

@inproceedings{lin2017focal,
  title={Focal loss for dense object detection},
  author={Lin, Tsung-Yi and Goyal, Priya and Girshick, Ross and He, Kaiming and Doll{\'a}r, Piotr},
  booktitle={Proceedings of the IEEE international conference on computer vision},
  pages={2980--2988},
  year={2017}
}

@article{he2024does,
  title={Does prompt formatting have any impact on llm performance?},
  author={He, Jia and Rungta, Mukund and Koleczek, David and Sekhon, Arshdeep and Wang, Franklin X and Hasan, Sadid},
  journal={arXiv preprint arXiv:2411.10541},
  year={2024}
}
